\pdfoutput=1

\documentclass[conference,letterpaper]{IEEEtran}

\usepackage[utf8]{inputenc} 
\usepackage[T1]{fontenc}
\usepackage{url}
\usepackage{ifthen}
\usepackage{cite}
\usepackage{comment}
\usepackage[normalem]{ulem}

\usepackage[cmex10]{amsmath} %

\usepackage[utf8]{inputenc} 
\usepackage[T1]{fontenc}
\usepackage{url}              %
\usepackage{cite}             %
\usepackage{pgfplots}
\pgfplotsset{compat=newest}

\usepackage[cmex10]{amsmath}  %
\usepackage{mleftright}       %
\mleftright                   %

\usepackage{graphicx}         %
\usepackage{booktabs}         %

\usepackage[ruled]{algorithm}
\usepackage{colortbl}
\usepackage{booktabs}
\usepackage{multirow}
\usepackage{caption}
\usepackage{float}
\usepackage{multicol}
\usepackage[normalem]{ulem}
\usepackage{algpseudocode}
\algnewcommand{\Initialize}{%
  \State \textbf{Initialize:}
}
\algnewcommand{\Output}{%
  \State \textbf{Output:}
}
\usepackage{mathtools}
\usepackage{graphicx}
\usepackage{url}

\usepackage[
    shortcuts,       %
    acronym,         %
    section=section, %
]{glossaries}
\usepackage{tikz}
\usetikzlibrary{backgrounds,calc,shadings,shapes.arrows,shapes.symbols,shadows,fit,positioning,topaths,hobby}
\usepackage{float}
\usepackage{amssymb}
\usepackage{amsmath}
\usepackage{amsthm}
\usepackage{xspace}
\usepackage[capitalize]{cleveref}
\usepackage{subfig}
\usepackage[inline]{enumitem}
\usepackage{cite}
\usepackage{siunitx}
\usepackage{dsfont}
\usepackage{balance}
\usepackage{mathabx}

\newtheorem{theorem}{Theorem}
\newtheorem{corollary}{Corollary}
\newtheorem{lemma}{Lemma}
\newtheorem{observation}{Observation}
\newtheorem{proposition}{Proposition}

\newtheorem{assumption}{Assumption}

\IEEEoverridecommandlockouts

\newcommand{\verts}{\ensuremath{\mathcal{V}}}
\newcommand{\edges}{\ensuremath{\mathcal{E}}}
\newcommand{\nnodes}{\ensuremath{n}}
\newcommand{\nodeidx}{\ensuremath{i}}

\newcommand{\degr}[1][i]{\ensuremath{\text{deg}(#1)}}

\newcommand{\hitk}[1][\rwidx]{\ensuremath{H_{i,j}(#1)}}
\newcommand{\hit}{\ensuremath{H_{i,j}}}
\newcommand{\selfhit}[1][\rwidx]{\ensuremath{H_{i,i}(#1)}}
\newcommand{\return}[1][\nodeidx]{\ensuremath{R_{\nodeidx}}}

\newcommand{\E}{\ensuremath{\mathrm{E}}}
\newcommand{\nwalks}[1][t]{\ensuremath{Z_{#1}}}
\newcommand{\nwalksrw}[1][t]{\ensuremath{\mathsf{Z}_{#1}}}
\newcommand{\rwidx}{\ensuremath{k}}
\newcommand{\rwidxtwo}{\ensuremath{\ell}}
\newcommand{\ntarget}{\ensuremath{Z_0}}
\newcommand{\lastseen}[1][\rwidx]{\ensuremath{L_{i, #1}(t)}}
\newcommand{\lastseenset}{\ensuremath{\mathcal{L}_{i}(t)}}
\newcommand{\pfork}{\ensuremath{p_{\text{fork}}}}

\newcommand{\returnrate}{\ensuremath{\lambda_r}}
\newcommand{\arrivalrate}{\ensuremath{\lambda_a}}
\newcommand{\activeidx}{\ensuremath{k}}
\newcommand{\activesym}{\ensuremath{\mathcal{A}}}

\newcommand{\terminatedsym}{\ensuremath{\mathcal{D}}}

\newcommand{\nest}{\ensuremath{\hat{\theta}_\nodeidx(t)}}

\newcommand{\rwest}[1][\rwidx]{\ensuremath{S(t-\lastseen[#1])}}
\newcommand{\rwestri}[1][\rwidx]{\ensuremath{S(\returnsample)}}

\newcommand{\termtime}{\ensuremath{{T_d}}}
\newcommand{\returncdf}[1][t]{\ensuremath{\hat{F}_{\return}(#1)}}

\newcommand{\pforkscale}{\ensuremath{p}}
\newcommand{\thres}{\ensuremath{\varepsilon}}
\newcommand{\thresterm}{\ensuremath{\varepsilon_2}}
\newcommand{\thresmp}{\ensuremath{\varepsilon_{\text{mp}}}}

\newcommand{\returnsample}{\ensuremath{r_i}}
\newcommand{\stablenrws}{\ensuremath{\nactive}}
\newcommand{\stabledelta}{\ensuremath{\delta^\prime}}
\newcommand{\define}{\ensuremath{:=}}
\newcommand{\forktime}{\ensuremath{T_f}}

\newcommand{\arrivaltime}{\ensuremath{T_a}}
\newcommand{\arrivaltimerv}{\ensuremath{\mathsf{T}_a}}

\newcommand{\forksym}{\ensuremath{\mathcal{F}}}

\newcommand{\epstmp}{\ensuremath{\varepsilon^\prime}}
\newcommand{\ncurrent}{\ensuremath{\nactive}}
\newcommand{\pforkcurrent}[1][\ncurrent]{\ensuremath{p_{#1}}}

\newcommand{\deltanoforkstable}{\ensuremath{\delta_1}}

\newcommand{\tnofork}[1][\deltanofork]{\ensuremath{T_N}}
\newcommand{\tallvisit}[1][\deltanofork]{\ensuremath{T_V}}
\newcommand{\tcover}[1][\deltanofork]{\ensuremath{T_C^\prime}}
\newcommand{\ttotal}{\ensuremath{T}}
\newcommand{\nactive}{\ensuremath{K}}
\newcommand{\nterm}{\ensuremath{D}}
\newcommand{\deltanofork}[1][\nterm]{\ensuremath{\delta_{#1}(\ttotal)}}
\newcommand{\deltanoforksol}[1][\nterm]{\ensuremath{\delta_{#1}}}

\newcommand{\uniform}{\ensuremath{\mathcal{U}}}
\newcommand{\nmax}{\ensuremath{m}}
\newcommand{\tnoforki}[1][\rwstep]{\ensuremath{T_{#1,2}}}
\newcommand{\tallvisiti}[1][\rwstep]{\ensuremath{T_{#1,1}}}

\newcommand{\nbound}{\ensuremath{z}}
\newcommand{\rwstep}{\ensuremath{\nu}}
\newcommand{\deltavisiti}{\ensuremath{\delta_{\rwstep, 1}}}
\newcommand{\deltadoubleforki}{\ensuremath{\delta_{\rwstep, 2}}}
\newcommand{\deltasumi}{\ensuremath{\delta_{\rwstep}}}

\newcommand{\forkcdfvar}{\ensuremath{x}}
\newcommand{\forkestrv}{\ensuremath{\hat{\theta}_{\forktime, \termtime}(t)}}
\newcommand{\forkcdfcond}[1][\forkcdfvar]{\ensuremath{F_{\forkestrv}(#1 \vert \arrivaltimerv \! = \! \arrivaltime)}}
\newcommand{\forkcdf}[1][\forkcdfvar]{\ensuremath{F_{\forkestrv}(#1)}}
\newcommand{\arrivalpdf}[1][t]{\ensuremath{f_{\arrivaltimerv}(#1)}}

\newcommand{\termtimes}{\ensuremath{\mathcal{T}_D}}
\newcommand{\forktimes}{\ensuremath{\mathcal{T}_F}}

\newcommand{\graph}{\ensuremath{\mathcal{G}}}
\newcommand{\algname}{\textsc{DecAFork}\xspace}
\newcommand{\algnameplus}{\textsc{DecAFork+}\xspace}
\newcommand{\missingperson}{\textsc{MissingPerson}\xspace}

\newcommand{\pfail}{\ensuremath{p_f}}
\newcommand{\ptransit}{\ensuremath{p_b}}
\newcommand{\lastevent}{\ensuremath{T_\ell}}

\newcommand{\tlbfork}[1][\nterm]{\ensuremath{T_{#1}}}
\newcommand{\tforkn}{\ensuremath{T_\nterm^{\recovered^\prime}}}
\newcommand{\recovered}{\ensuremath{R}}
\newcommand{\deltasum}{\ensuremath{\delta_\Sigma}}

\newcommand{\history}[1][t]{\ensuremath{\mathcal{H}_{#1}}}
\newcommand{\pforkt}[1][t]{\ensuremath{\pfork(\history[#1])}}
\newcommand{\nrwtmp}{\ensuremath{\kappa}}
\newcommand{\forkboundt}[1][t]{\ensuremath{\bar{p}_{\text{fork}}(\history[#1])}}
\newcommand{\ptermt}[1][t]{\ensuremath{\bar{p}_{\text{term}}(\history[#1])}}

\newcommand{\firstfork}{\ensuremath{t_0}}
\newcommand{\forkidxtmp}{\ensuremath{\rwstep}}

\makeatletter
\newcommand\fs@spaceruled{\def\@fs@cfont{\bfseries}\let\@fs@capt\floatc@ruled
  \def\@fs@pre{\vspace{.22cm}\hrule height.8pt depth0pt \kern2pt}%
  \def\@fs@post{\kern2pt\hrule\relax}%
  \def\@fs@mid{\kern2pt\hrule\kern2pt}%
  \let\@fs@iftopcapt\iftrue}
\makeatother

\newif\iffigures
\figurestrue

\begin{document}
\title{Self-Regulating Random Walks for Resilient Decentralized Learning on Graphs \vspace{-0.0cm}}

\author{%
  \IEEEauthorblockN{Maximilian~Egger\IEEEauthorrefmark{1},
                    Rawad~Bitar\IEEEauthorrefmark{1},
                    Ghadir~Ayache\IEEEauthorrefmark{2},
                    Antonia~Wachter-Zeh\IEEEauthorrefmark{1}
                    and Salim~El~Rouayheb\IEEEauthorrefmark{3}}
  \IEEEauthorblockA{\IEEEauthorrefmark{1}%
                   Technical University of Munich, Munich, Germany \{maximilian.egger, rawad.bitar, antonia.wachter-zeh\}@tum.de}
  \IEEEauthorblockA{\IEEEauthorrefmark{2}%
                       Etsy Inc, New York, USA,
                    ayache.ghad@gmail.com}
  \IEEEauthorblockA{\IEEEauthorrefmark{3}%
                    Rutgers University, New Brunswick, USA,
                    salim.elrouayheb@rutgers.edu\vspace{-.5cm}}
    \thanks{This project has received funding from the German Research Foundation (DFG) under Grant Agreement Nos. BI 2492/1-1 and WA 3907/7-1. The work of S. El Rouayheb  was supported in part by the Army Research
Lab (ARL) under Grant W911NF-21-2-0272, and in part by
the National Science Foundation (NSF) under Grant CNS2148182. Preliminary results of this work were presented at IEEE GLOBECOM 2024 \cite{egger2024self}.}
}

\maketitle

\begin{abstract}
 Consider the setting of multiple random walks (RWs) on a graph executing a certain computational task. For instance, in decentralized learning via RWs, a model is updated at each iteration based on the local data of the visited node and then passed to a randomly chosen neighbor. RWs can fail due to node or link failures. The goal is to maintain a desired number of RWs to ensure failure resilience. Achieving this is challenging due to the lack of a central entity to track which RWs have failed to replace them with new ones by forking (duplicating) surviving ones. Without duplications, the number of RWs will eventually go to zero, causing a catastrophic failure of the system. We propose two decentralized algorithms called \algname and \algnameplus that can maintain the number of RWs in the graph around a desired value even in the presence of arbitrary RW failures. Nodes continuously estimate the number of surviving RWs by estimating their return time distribution and fork the RWs when failures are likely to happen. \algnameplus additionally allows terminations to avoid overloading the network by forking too many RWs. We present extensive numerical simulations that show the performance of \algname and \algnameplus regarding fast detection and reaction to failures compared to a baseline, and establish theoretical guarantees on the performance of both algorithms.

\end{abstract}

\section{Introduction}

Decentralized systems consist of a collection of users who cooperate to accomplish a given task without the need for a central node coordinating the process. Users are modeled by vertices in a graph, and communication links by edges. %
Random walks on graphs are considered efficient and scalable decentralized solutions. A virtual token decides which user can run computations. The user holding the token runs local computations, updates the task, and passes the task and the token to one of its randomly chosen neighbors. 
This repeats until the task is declared accomplished. The motivating example is decentralized learning \cite{Johansson2009ARI, Nedic2009DistributedSM, sun2018markov, Ayache2019RandomWG}, where the task is to train a neural network on the union of the users' data. The user holding the token runs local iterations, updates the model, and passes it to a neighbor until some convergence criterion is met. When moving, the token draws a \emph{random walk} over the graph.

Random walks (RWs) have seen a surge of interest in various fields, such as networking \cite{johansson2007simple, johansson2010randomized}, optimization \cite{ram2009incremental,duchi2012ergodic, sun2018markov,ayache2021,ayache2022walk,liu2024entrapment} and signal processing \cite{rabbat2005quantized, zhao2014asynchronous,mao2020walkman}. 
Our main motivation for studying RWs stems from their application in decentralized learning, although their applicability goes beyond to various other fields. RW-based learning algorithms are proposed as a communication-efficient alternative to Gossip algorithms \cite{Boyd2005GossipAD, Shah2009GossipA,Lu2010GossipAF,boyd2006randomized,nedic2009distributed,koloskova2019decentralized,duchi2011dual} that require every node to run local computations and broadcast their updated model to all their neighbors. A trade-off between the advantages of RWs and Gossip-approaches has been studied recently in \cite{gholami2024digest}. RW algorithms are key in distributed communication \cite{Maddah_gossip_RW_for_local_route_discovery,sarma2015efficient}, notably for decentralized model updates in distributed learning \cite{Johansson2009ARI, Nedic2009DistributedSM, ayache2021}. 
However, while RW algorithms can score with their efficiency, they lack failure resilience. It is enough for the node, or the communication link, holding the token to fail to cause a \emph{catastrophic failure}; losing all progress made thus far. Guaranteeing resilience against catastrophic failures is, hence, paramount to decentralized settings using RWs. %
Existing research has primarily addressed common anomalies like delays \cite{wooter_gossip_delays}. 
To the best of our knowledge, the challenging task of guaranteeing adaptive resilience to failures has not yet been studied in the literature despite recent progress on RW-based learning algorithms. We propose the first solution to this problem.

The main challenge in guaranteeing failure tolerance is the absence of a central entity orchestrating the process. A naive solution is to run multiple RWs in parallel. However, since no central entity is tracking the RWs, after certain failures occur, all RWs may fail, leading to a catastrophic failure. Hence, a decentralized mechanism allowing the nodes to dynamically and independently adjust the number of RWs in the system is needed. This motivates us to pose this question: \emph{how to devise decentralized control algorithms that can maintain, at any time, a desired} level of redundancy \emph{(multiple RWs running in parallel) in the system by quickly reacting to failures and dynamically creating RWs without making any a priori assumptions about the statistics of failures?}  

Control algorithms can affect the communication and computation overhead, therefore we impose the following rules:

\noindent{\bf Rule 1:} No central entity can observe and communicate with all the nodes in the graph. Moreover, nodes can only communicate with their neighbors.  \\
\noindent{\bf Rule 2:} RWs cannot directly communicate with each other\footnote{The nodes currently visited by the RWs cannot directly communicate with each other outside the graph. Otherwise,  the problem can be easily solved by letting the RWs regularly ping each other to indicate that they are still alive. }. \\
 \noindent{\bf Rule 3:}  An RW can be forked (creating an independent duplicate copy) or terminated by the currently visited node.

A straightforward solution here can be to let each node independently fork (create a duplicate copy of) an RW after a prescribed time $T$ to replace any possible failure that may have occurred. While such an algorithm satisfies our three rules above, it has the following undesired drawback that we want to avoid: either the network is flooded with an ever-increasing number of RWs (for small $T$) or all the RWs eventually fail (for large $T$). That is because the failures are arbitrary with no assumption on their statistics that could have helped with designing $T$. Our objective is to design a decentralized algorithm that can avoid these two extreme cases\footnote{Evidently, if all the RWs fail at the same time, this would be a catastrophic event that no algorithm can recover from.} and guarantees over time a constant number of RWs in the graph. While the main motivation for this work is decentralized learning, we design an algorithm that can be widely applied to any RW-based system.

\emph{Contributions:} In this paper, we propose \algname and \algnameplus, two novel \underline{dec}entralized (randomized) \underline{a}lgorithms \underline{fork}ing RWs, capable of detecting and reacting to deviations from a desired level of redundancy. %
Our algorithms satisfy Rules~1-3 %
and do not assume any specific failure model. %
The key component of any desirable algorithm is a fast reaction time representing the time the algorithm takes to bring the system back to the desired number of RWs after failures occur (see \cref{fig:catastrophic_failures}), and avoiding exceeding the desired redundancy. We show that these objectives can be contradictory and that our algorithms provide a satisfying trade-off. 
A key ingredient of our algorithms is a distributed estimator of the number of surviving random walks that can be implemented at each node to enable decentralized control decisions in the form of forks and terminations. 
We derive theoretical guarantees on the reaction time and the redundancy level by thoroughly analyzing the time-dependent properties of the proposed estimator used in our algorithms and the interplay with the graph setting at hand. We show through extensive numerical experiments that the number of RWs in the graph %
is maintained around a desired value when challenged with a variety of threat models and different graphs.

\section{System Model} \label{sec:system_model}
{\em Graph: } We consider a system of $\nnodes$ nodes that want to collaborate  in a decentralized fashion  to run a certain  computational task. Each node possesses local data and can communicate only with a subset of the other nodes. We model this system by an undirected graph $\graph = (\verts, \edges)$ with $\nnodes$ vertices $\verts = [\nnodes] \define \{1, \dots, \nnodes\}$ representing the nodes and the set of edges (links) $\edges$ comprising of all pairs of nodes $(i, j)$, $(j, i)$ that can share information.  We will assume that graph \graph\ is  connected\footnote{ If the graph is not connected, our study can be applied separately to each connected component.} making the resulting RWs irreducible \cite{Levin2017MarkovCA}.   The degree $\degr$ of a node $i \in \verts$ is defined as the number of its neighbors, i.e., $\degr = \vert \{(i, j) \in \edges\} \vert$.

{\em Random Walks: }We are interested in tasks that can be accomplished by an RW carrying a token message on the graph. At discrete time steps, only the node holding the token (the currently visited node) may run some computation and update the message. Then, it passes the updated token message to one of its neighbors chosen randomly according to a fixed transition matrix $P$. We consider simple random walks, where the node to whom a node sends the token is selected uniformly at random from all neighbors. This is repeated until certain stopping criteria are met. We abuse nomenclature and refer to the token as the Random Walk (RW). Since we consider multiple RWs, each RW, indexed by $k$, is distinguishable by the nodes they visit through a unique identifier.   %

{\em Failures of Random Walks: } Random walks on the graph may fail arbitrarily. Potential reasons include that the nodes or links are temporarily down leading to the loss of the token when being passed to the next node. Similarly, a random walk can fail due to queue or buffer overflows, or the node currently holding the token may face processing issues, making it impossible to pass on the token. We seek algorithms that dynamically adapt to any number of failures occurring over time. Therefore, for the algorithm design, we do not impose any assumptions on the failure model. To validate our algorithms, we consider three different threat models: 1) a burst failure happens where multiple RWs fail simultaneously, 2) RWs fail independently of each other with probability $\pfail$ every time they visit a node, and 3) there exists a Byzantine node modeled by a two-state Markov chain with transition probability $\ptransit$, which deterministically terminates all incoming random walks when being in one state (Byz), and behaves according to the protocol otherwise (No Byz). We show the performance of our algorithms for failure models 1)-3) in \cref{fig:catastrophic_failures,fig:probabilistic_failures,fig:byzantine_failures}, respectively. We choose those failure models as they cover a large class of common failures in decentralized systems, e.g., \cite{palmieri2024robustness}.

{\em Forking: }To avoid catastrophic failures, an RW is allowed to be duplicated in the following way. After updating the RW, the currently visited node will make a duplicated \emph{identical copy} of the RW that is propagated on the graph according to the transition matrix $P$ independently of the original one. %

{\em Terminating: } To improve the reaction time to failure events while limiting the redundancy, currently visited nodes can terminate the visiting RW. Such termination is final and involves complete loss of the information held by the RW. 

Subsequently, let $\nwalksrw$ be the random variable modeling the number of RWs at time $t$, and let $\nwalks$ be its realization at time $t$.
The objective is to design a fully decentralized algorithm that abides by Rules 1-3 and satisfies the following objectives.
\begin{itemize}
    \item \textbf{Stability:} The number of RWs $\nwalksrw$ is maintained around a desired target $\ntarget$, i.e., $\Pr(\vert \nwalksrw - \ntarget \vert < \epsilon) \approx 1$ for a small $\epsilon\geq 0$, even after failures occur. The redundancy reflected by $\nwalks$ does not significantly exceed $\ntarget$, which can be especially challenging after failure events.
    \item \textbf{Resilience:} At least one RW maintains activity at all times and particularly after a failure event.
    \item \textbf{Reaction:} Rapid detection and reaction times to failure events, enabling the system to tolerate and rapidly adapt to future failure events.
\end{itemize}
The challenge is to achieve a trade-off between frequently forking RWs and flooding the network, or seldomly forking RWs and risking a catastrophic failure.

For the optimal functioning of our algorithms and the validity of our theoretical results, all $\ntarget$ random walks should have been active for long enough to have visited each node at least once before the first failure of an RW. Apart from this assumption, our algorithms do not rely on any assumptions about the graph structure. For a tractable theoretical analysis, we will later introduce some required assumptions in \cref{assumption:distributions}.

\section{Main Results: \algname \& \algnameplus} \label{sec:main_results}

In this section, we first provide a simple and intuitive algorithm, termed \missingperson, that will serve as a baseline for our proposed algorithms and that will illustrate the problem and the challenges faced in algorithmic designs for failure resilience in RW methods. We introduce and examine \algname, a forking-only based method that serves as a building block for \algnameplus that additionally utilizes deliberate terminations. Numerical experiments are provided to show the difference between the algorithms. The theoretical analysis of the algorithms follows in \cref{sec:theory_decafork,sec:theory_algplus}.

\floatstyle{spaceruled}%
\restylefloat{algorithm}
\newlength{\textfloatsepsave}
\setlength{\textfloatsepsave}{\textfloatsep}
\setlength{\textfloatsep}{-10pt}
\begin{algorithm}[t!]
\caption*{\textbf{\missingperson}: Executed when RW $\rwidx$ visits node $\nodeidx$ at $t$}\label{alg:missingperson}
\begin{algorithmic}
\Require $\thres$, $\ntarget$, $k$, $t$, $\forall \rwidxtwo \in [\ntarget]: \lastseen[\rwidxtwo]$, %
\State Update $\lastseen \gets t$
\For{$\rwidxtwo \in [\ntarget] \setminus \{\rwidx\}$}
\If{$t-\lastseen[\rwidxtwo]>\thresmp$}
\State Fork RW $\rwidx$ with probability $\pforkscale = \frac{1}{\ntarget}$ 
\State with identifier $\rwidxtwo$ in replacement of RW $\rwidxtwo$
\EndIf
\EndFor
\end{algorithmic}
\end{algorithm}
\setlength{\textfloatsep}{\textfloatsepsave}

\subsection{\missingperson\ - A Baseline Method} Assume we start with $\ntarget$ RWs on the graph. The goal is to maintain $\nwalksrw$ around $\ntarget$ by replacing the initial RWs when they are deemed missing. Each node $\nodeidx \in [\nnodes]$ initializes for each RW $\rwidxtwo \in [\ntarget]$ a variable $\lastseen[\rwidxtwo]$ that tracks the last time the RW was visiting, i.e., {\color{green!50!black}$\forall \rwidxtwo \in [\ntarget]: \lastseen[\rwidxtwo] = 0$}. Each time an RW $\rwidx$ visits node $\nodeidx$, the node checks for all other RWs $\rwidxtwo \in [\ntarget]\setminus \{\rwidx\}$ the time difference since when the RWs were seen last. If any RW $\rwidxtwo$ was not seen for more than $\thresmp$ timesteps, it deems the RW missing and forks a new one with probability $\pforkscale$ in replacement with the same identifier $\rwidxtwo$. Intuitively, this approach is to ensure that for each initial RW at least one is maintained active in the network. We term this approach \missingperson and summarize the algorithm above.

The challenge in such algorithms is that the distribution of inter-arrival times depends on the graph and on the location of the nodes within the graph, making finding a good threshold $\thresmp$ a challenging task. With an appropriately chosen $\thresmp$, the risk that more than one RWs is forked in replacement for a missing one remains. We will next present an algorithm that alleviates this drawback, and significantly improves the resilience against failures.

\subsection{\algname: Robustness by Careful Forking}
We introduce \algname, a decentralized algorithm that forks RWs probabilistically to avoid catastrophic failures without flooding the network with RWs.
\algname works as follows. The network starts with a target number $\ntarget$ of RWs\footnote{This can be instantiated by one node creating $\ntarget$ RWs in the beginning.}. At every time instance, when a node receives an RW, it estimates the number of \emph{active} RWs in the network. If this number is too low, the node forks the visiting RW with probability $p = 1/\ntarget$ so that, on average, one RW is forked at a given time.

To estimate the number of active RWs in the system, we rely on the RW's hitting time and return time. The \emph{hitting time} $\hitk$ of an RW $\rwidx$ at node $\nodeidx$  is the random variable that denotes the first time the RW gets to node $i$ starting from node $j$. Let $v_t^{(\rwidx)}$ be the node visited by RW $\rwidx$ at time $t$, we can define the hitting time as $\hitk = \min \{ t \in \mathbb{N}, v_t^{(\rwidx)} = i \,|\, v^{(\rwidx)}_0 = j, v_1^{(\rwidx)},\dots, v_{t-1}^{(\rwidx)} \neq i \} $. Furthermore, the random variable $\return(\rwidx) = \selfhit$ describes the first return time of a random walk to node $\nodeidx$ after leaving node $\nodeidx$.

\floatstyle{spaceruled}%
\restylefloat{algorithm}
\setlength{\textfloatsepsave}{\textfloatsep}
\setlength{\textfloatsep}{-10pt}
\begin{algorithm}[t!]
\caption*{\textbf{\algname}: Executed when RW $\rwidx$ visits node $\nodeidx$ at time $t$}\label{alg:strategy}
\begin{algorithmic}
\Require $\thres$, $\ntarget$, $k$, $t$, $\lastseenset$, $\forall \rwidxtwo \in \lastseenset: \lastseen[\rwidxtwo]$, %
\If{$\rwidx \in \lastseenset$} 
\State \hspace{.05cm} Add $t-\lastseen$ as sample for the distribution of $\return$
\State \hspace{.05cm} Update $\lastseen \gets t$
\Else \hspace{.05cm} Create $\lastseen = t$
\State $\lastseenset\gets \lastseenset \cup \{\rwidx\}$
\EndIf
\State Create $\nest \triangleq \frac{1}{2}$, an estimate of the number of RWs
\For{$\rwidxtwo \in \lastseenset\setminus \{\rwidx\}$}
\State Calculate the survival probability $\rwest[\rwidxtwo]$ 
\State of the return time $\return$
\State Update $\nest \gets \nest + \rwest[\rwidxtwo]$
\EndFor
\If{$\nest<\thres$}
\State Fork RW $\rwidx$ with probability $\pforkscale = \frac{1}{\ntarget}$
\State with new unique identifier $\rwidx^\prime$
\EndIf
\State \textbf{Return} $\nest$
\end{algorithmic}
\end{algorithm}
\setlength{\textfloatsep}{\textfloatsepsave}

We assume that RWs walk on the graph independently. Hence, they have independent and identically distributed hitting and return times. Therefore, we focus on one random variable $\return$ and drop the dependency on $\rwidx$. We do the same for $\hit$. To have a refined estimate for $\return$ and its distribution, every empirically observed return time for every RW will contribute to the estimate of the random variable $\return$. Hence, the algorithm requires an initialization phase without RW failures. This allows the $\ntarget$ initial RWs to circulate the graph sufficiently for the nodes to have reasonable estimates of the return time distribution. A sufficient time is at least until each RW visits each node at least once.

To measure the empirical distribution of the return time $\return$, each node $\nodeidx$ keeps track of the time it has last seen RW~$k$, denoted by the random variable $\lastseen$. This variable is created as $\lastseen = t_1$ at time $t_1$ when RW $k$ first hits node $i$. Then, at time $t$, $\lastseen$ is updated as $\lastseen = t$ if and only if RW $k$ visits node $\nodeidx$ at time $t$. 
When an RW $\rwidx$ visits node~$\nodeidx$ at time $t$, before updating $\lastseen$, node $\nodeidx$ measures a sample of $\return$ by computing $t-\lastseen$ to build an empirical distribution\footnote{The empirical distribution $\rwest$ can be replaced with the analytical survival function to speed up the initialization phase and the algorithm's precision. Such results are known, e.g., for random regular graphs \cite{tishby2021analytical}. In this case, it is sufficient that each RW has visited each node at least once before the first failure.} of $\return$. Let $\returncdf$ be the established empirical cumulative distribution function (CDF) of the return time of an RW at node $\nodeidx$ on the graph $\graph$. We call the distribution $\rwest \define 1-\returncdf[t-\lastseen]$ \emph{survival function}, denoting the estimated probability of RWs returning after time $t$, i.e., $\hat{\Pr}(\return>t-\lastseen)$. 

Since $\lastseen$ is a random variable, the function $\rwest$ representing the survival probability is also a random variable. %
Using $\return$, each node $\nodeidx$ maintains an estimate of $\nwalks$. %

\begin{figure}[!t]
    \centering
    \iffigures
    \resizebox{.9\linewidth}{!}{\input{plots/catastrophic_failures}}
    \fi
    \caption{\small Performance of \missingperson, \algname and \algnameplus in maintaining the number of random walks (RWs) $Z_t$ around a desired value $Z_0=10$. $\graph$ is a random $8$-degree regular graph with $\nnodes = 100$ nodes. 
    We induce two burst failure events at $t=2000$ and $t=6000$ where multiple RWs fail simultaneously. \missingperson over-reacts to the failure events by over-forking, whereas \algname ($\thres=2$) reacts faster and only forks RWs until $Z_t$ stabilizes around $Z_0$. \algnameplus ($\thres=3.25, \thresterm=5.75$) can react faster by terminating RWs if $Z_t$ exceeds $Z_0$. The Standard deviations over $50$ simulation runs are depicted by shaded areas.}
    \label{fig:catastrophic_failures} \vspace{-.5cm}
\end{figure}

\emph{Forking strategy: }In \algname, only nodes visited by an RW can fork the visiting RW\footnote{If multiple RWs visit a node, the node chooses one of them and follows the detailed steps.}. Let node $\nodeidx$ be visited by RW $\rwidx$ at time $t$. Let $\lastseenset$ be the set of indices of RWs that have visited node $\nodeidx$ until time $t$. To estimate $\nwalks$, node $\nodeidx$ computes%
\begin{equation} \label{eq:estimator}
\vspace{-0.1cm}
    \nest =  1/2+\sum_{\ell\in \lastseenset\setminus\{\rwidx\}}\rwest[\rwidxtwo].
    \vspace{-0.1cm}
\end{equation} As we show in the sequel, the value of $\nest$ serves as an estimate of $\nwalks/2$. For a predetermined parameter $\thres>0$, if $\nest < \thres$, the node declares that $\nwalks<\ntarget$. To avoid flooding the network, i.e., avoiding that all nodes fork simultaneously, node $\nodeidx$ forks\footnote{Forked RWs behave immediately like active ones leaving the forking node.} RW $\rwidx$ with probability $p= 1/\ntarget$. To distinguish the random walks, each RW, even the forked one, is given its own unique identifier.\footnote{Such an identifier can be the index of the original random walk, i.e., $\rwidx \in [\nwalks[0]]$, at the beginning. When a node $\nodeidx$ forks a random walk at time $\forktime$, it appends its own identifier and the time $\forktime$ of forking.} We summarize \algname in the algorithm above.

\begin{figure}[!t]
    \centering
    \iffigures
    \resizebox{.9\linewidth}{!}{\input{plots/probabilistic_failures_comp}}
    \fi
    \caption{\small Performance of \algname and \algnameplus for a random $8$-degree regular graph with $\nnodes = 100$ nodes. In addition to two burst failure events at $t=2000$ and $t=6000$, each RW can independently fail with probability $\pfail$ at each time step. The parameters $\thres$ and $\thresterm$ are chosen to stabilize around $\ntarget$ for $\pfail=0$, i.e., as in \cref{fig:catastrophic_failures}. \algnameplus exhibits a stable performance for different values of $\pfail$. \algname can successfully recover from burst failures but does not attain the target redundancy of $\ntarget$ due to continuous probabilistic failures that outweigh the forks.}
    \label{fig:probabilistic_failures} \vspace{-.6cm}
\end{figure}

\emph{Choosing the threshold: }The difficulty lies in designing the parameters $\thres$ and $\pforkscale$, which should facilitate both \begin{enumerate*}[label={\textit{(\roman*)}}] \item\label{item:early_detection} early detection of failures and consequently forking of RWs; and \item\label{item:avoid_flood} avoiding forking when the number of walks is $\ntarget$ or above.\end{enumerate*} We establish in \cref{prop:dist_active} (\cref{sec:theory_decafork}) the distribution of $\nest$ for the desired case of $\nactive = \ntarget$ infinitely long active RWs in the form of its CDF $F_{\Sigma_{\stablenrws-1}}(\sigma)$. The "$-1$" is because one RW is visiting; hence, its contribution to $\nest$ is not probabilistic. The parameter $\thres$ can be conveniently chosen based on this result: Let $F_{\Sigma_{\ntarget-1}}(\thres-\frac{1}{2})$ be the probability of estimating at most $\thres$ active RWs assuming that $\ntarget$ RWs are active. Intuitively, this reflects the likelihood of the assumption of $\ntarget$ active RWs being accurate. %
We chose the value of $\thres$ such that the probability of forking with $\ntarget$ active RWs is negligible. 
Let $\stabledelta \define F_{\Sigma_{\ntarget-1}}(\thres-\frac{1}{2})$. According to \algname, with $\ntarget$ active RWs, a node forks with probability $\pfork = \pforkscale \cdot \stabledelta$. If $\pforkscale = 1$, a node deterministically forks once it encounters a large deviation from the expected value of $\nest$, exhibiting a fast reaction to failures. The probability of forking vanishes for more than $\ntarget$ active RWs, which is a desirable property.

The parameter $p$ is chosen to avoid flooding the network. In cases where a failure happens and $\nactive$ random walks remain active, at each time step, $\nactive$ nodes will simultaneously be able to detect the failures and decide to fork a new RW. Scaling the probability of forking at each node by $\pforkscale = 1/\ntarget$ avoids having too many forks at the expense of increasing the reaction time to failure events. %

Note that \algname can be run on any connected graph and does not make any assumption on the distribution of the return time $\return$, which is estimated on the fly. We verified this fact through numerical simulations on different families of graphs, such as complete graphs, Erdos Rényi graphs, and Power Law graphs. The results can be found in \cref{sec:add_experiments}.
We plot in \cref{fig:catastrophic_failures,fig:probabilistic_failures,fig:byzantine_failures} the performance of \missingperson, \algname and \algnameplus (which we introduce in the following) for the three failure models described in \cref{sec:system_model}. Since \missingperson provides suboptimal results for the easiest case of only one-time failures (cf. \cref{fig:catastrophic_failures}), we omit it from the plots where additional failures are introduced.

\subsection{\algnameplus\ : Faster Reaction by Deliberate Termination}
It can be seen from \cref{fig:probabilistic_failures} that \algname does not perform as desired for probabilistic failures with relatively large failure probability $\pfail$, and fails to cope against Byzantine nodes when the threshold $\thres$ remains unchanged (cf. \cref{fig:byzantine_failures}). When increasing the threshold to successfully recover from Byzantine failures, \algname overshoots in the case that the Byzantine node is in the state not to terminate (No Byz). We introduce an extension to \algname that overcomes these challenges, which we term \algnameplus. The difference is to additionally allow nodes to terminate RWs on purpose in cases where they believe the target number of RWs $Z_0$ is exceeded. The termination strategy is similar to the forking strategy in \algname.
When node $\rwidx$ visited by an RW $\rwidx$ has $\nest$ larger than a chosen threshold $\thresterm$, the node probabilistically terminates the RW with probability $\pforkscale$. Adding this component to \algname allows for much faster reaction times since the forking threshold $\thres$ can be chosen larger than in \algname without risking overshooting significantly beyond $Z_0$. Deliberate terminations of the nodes will limit the number of RWs from above. 
The threshold $\thresterm$ can be chosen analogously to $\thres$: by \cref{prop:dist_active} (cf. \cref{sec:theory_decafork}), we directly get a distribution for the survival function of $\nest$, i.e., $1-F_{\Sigma_{\ntarget-1}}(\sigma)$ for an RW $\rwidx$ visiting node $\nodeidx$. We choose $\thresterm$ such that terminating an RW with having $\ntarget$ active RWs is negligible, i.e., $1-F_{\Sigma_{\ntarget-1}}(\thresterm-\frac{1}{2}) \approx 0$.

\floatstyle{spaceruled}%
\restylefloat{algorithm}
\setlength{\textfloatsepsave}{\textfloatsep}
\setlength{\textfloatsep}{-10pt}
\begin{algorithm}[!t]
\caption*{\textbf{\algnameplus}: Executed when RW $\rwidx$ visits node $\nodeidx$ at time $t$}\label{alg:algplus}
\begin{algorithmic}
\Require $\thres$, $\thresterm$, $\ntarget$, $k$, $t$, $\lastseenset$, $\forall \rwidxtwo \in \lastseenset: \lastseen[\rwidxtwo]$, %
\State $\nest = \algname\left(\thres, \ntarget, k, t, \lastseenset, \{\lastseen[\rwidxtwo]\}_{\rwidxtwo \in \lastseenset}\right)$
\If{$\nest>\thresterm$}
\State Terminate RW $\rwidx$ with probability $\pforkscale = \frac{1}{\ntarget}$
\EndIf
\end{algorithmic}
\end{algorithm}
\setlength{\textfloatsep}{\textfloatsepsave}

Our approach in \algname and \algnameplus is less sensitive to the actual return time distribution than simpler forking policies such as the \missingperson, where a node compares the time since an RW last visited to a threshold, above which a new RW is forked probabilistically.  This simplifies the parameter selection and improves the stability of the approach to different graph topologies. In all experiments, \algname and \algnameplus significantly outperformed alternative forking decision rules. Other algorithms we investigated either led to a massive blow-up of $\nwalks$, hence to a network overload, or could not cope with arbitrary failure events. We showcase in \cref{fig:catastrophic_failures,fig:probabilistic_failures,fig:byzantine_failures} the performance of our algorithms for different classes of failures for $8$-degree random regular graphs with $\nnodes=100$. From \cref{fig:catastrophic_failures}, it can be seen that both \algname (with $\thres=2$) and \algnameplus (with $\thresterm=5.75$) provide good stability in the presence of burst failures. Even for a properly tuned $\thresmp$, \missingperson exhibits a slower reaction time and significant overshoot beyond $\ntarget$. Notably, \algnameplus reacts significantly faster than \algname. In \cref{fig:probabilistic_failures}, we additionally introduce probabilistic failures at each iteration with probability $\pfail \in \{0.001, 0.0002\}$. While \algname recovers from burst failures, it does not attain the target redundancy after failure events. \algnameplus, due to its more competitive forking threshold $\thres=3.25$, achieves faster reaction times and successfully maintains $\nwalks$ around $\ntarget$. Only for byzantine failures (Byz), where a node terminates all visited RWs, \algname fails with $\thres=2$, but can cope with the byzantine behavior for $\thres=3.25$. However, the algorithm overshoots when the Byzantine node suddenly stops terminating RWs (No Byz). \algnameplus can successfully cope with these challenges. However, \algnameplus also introduces more fluctuations in the system, which can be undesired in certain cases where the lifetime of RWs should be maximized. Hence, both algorithms are of interest. Note that the design of \algname and \algnameplus does not rely on any assumption on the failure model, making them robust against a large class of failures. Since \missingperson is unsuccessful in the setting of only burst failures, which is the easiest among all considered here, we show the results for the latter two failure models only for \algname and \algnameplus. By ensuring at least one RW maintains activity at all times, from a learning perspective, the system behaves like that of a single RW without failures.

\begin{figure}[!t]
    \centering
    \iffigures
    \resizebox{.9\linewidth}{!}{\input{plots/byzantine_failures}}
    \fi
    \caption{\small Performance of \algname and \algnameplus for a random $8$-degree regular graph with $\nnodes = 100$ nodes. In addition to two burst failure events at $t=2000$ and $t=6000$, one dedicated node deterministically fails each incoming RW (Byz). The challenge is to cope with Byzantine nodes that can suddenly stop terminating RWs, i.e., behaving honestly (No Byz). Only \algnameplus can cope with this extreme failure model.}
    \label{fig:byzantine_failures} \vspace{-.5cm}
\end{figure}

\section{Analyzing \algname: Reaction vs. Overshoot} \label{sec:theory_decafork}

\algname and \algnameplus exhibit good performance regarding the reaction to failure events with an appropriate choice of $\thres$ and $\thresterm$. For tractability, we focus on analytically analyzing the detailed behavior of \algname in the sequel. In \cref{sec:theory_algplus}, we provide an additional result and an interpretation for \algnameplus. 
We first investigate the average of the estimation $\rwest$ in the case of a stable number of active random walks $\nactive$ on the graph (cf. \cref{prop:unbiased_estimator}), %
and then generalize this result to capture cases where RWs are forked and terminated at arbitrary times (cf. \cref{thm:unbiasedness}). We show the tension between the reaction time to failures and the probability of reaching $\nwalks>\ntarget$ for any $t>0$ immediately after the start of the algorithm (cf. \cref{prop:lower_bound,thm:upper_bound}) and provide a bound on the overshoot of the number of RWs after (burst) failure events (cf. \cref{thm:overshoot}), which requires a bound on the forking probability at any time $t$ (cf. \cref{thm:fork_bound}). %
Lengthy proofs are deferred to the appendix.

For our analysis, we require knowledge of the full distribution of the return time $\return$ and the first hitting time $\hit$ of RWs. The literature on such distributions is rare, with some exceptions. In fact, our assumptions\footnote{Note that with minor adaptation, our analysis can similarly be carried out for other types of return and first-hitting time distributions.} are justified by recent results on the distribution of first return and first hitting times of RWs on random regular graphs \cite{tishby2021analytical,tishby2022analytical}. It is shown that both $\return$ and $\hit$ exhibit a behavior similar to a geometric distribution. For $\return$, retroceding trajectories, i.e., those that return to node $\nodeidx$ the same way they left, affect the distribution for small realizations of $\return$. This relies on combinatoric arguments that we will neglect for tractability purposes. The overall dominating part of the distribution of $\return$ stems from non-retroceding trajectories, which provably yield an exponential behavior \cite{tishby2021analytical}. This is underlined by our experiments, where random regular graphs exhibit a distribution that can be well-approximated by a properly parameterized geometric distribution.

For the theoretical analysis, we study a continuous relaxation thereof, which serves as a proxy for the discrete random variables due to the nature of the discrete time steps. With this assumption, we can precisely describe the distribution of the estimation $\nest$. We choose an exponential distribution since it facilitates a rigorous and tractable analysis of the algorithm's performance while being the continuous analog to the geometric distribution. However, it is worth noting that our analysis does not rely on the memoryless property of the exponential distribution and can, hence, be generalized to different types of distributions whose CDF is invertible, following the same methodology.

\begin{assumption} \label{assumption:distributions}
    Based on the discussion above, we make the following assumptions for our theoretical analysis:
    \begin{itemize}
        \item The return time of RWs is independently and identically distributed according to $\return \sim \exp(\returnrate)$.\footnote{This assumption can be relaxed to capture different distributions for different nodes.}
        \item The first hitting time $\hit$ for random nodes $\nodeidx$ and $j$ of a forked RW $\rwidx$ is distributed according to $\hit \sim \exp(\arrivalrate)$.
    \end{itemize}
\end{assumption}

Knowing the analytical survival function, for the following analysis we replace $\rwest=1-\returncdf[t-\lastseen]$ by the exact distribution, i.e., $1-F_{\return}(t-\lastseen)$.

\subsection{On the Average of the Estimator $\nest$}

Before analyzing the estimator $\nest$ in detail, we verify that $\nest$ resembles the number of active random walks $Z_t$ given that all RWs are active for infinitely long\footnote{This assumption is technical and results from the exponential, unbounded distribution; it can be relaxed by truncating.}.

\begin{proposition} \label{prop:unbiased_estimator}
    Under \cref{assumption:distributions} and replacing the empirical distribution of $\return$ by its analytical counterpart, the estimator $\nest$ satisfies $2\E[\nest] = Z_t$ for infinitely long active random walks, i.e., no forks and terminations.
\end{proposition}

\begin{proof}
For any RW $\rwidxtwo$ and any $t$, let $\returnsample \define t-\lastseen[\rwidxtwo]$ be the time passed since RW $\rwidxtwo$ was last seen at node $\nodeidx$. 
Hence, we have $\rwest[\rwidxtwo] = \Pr(\return > t-\lastseen[\rwidxtwo]) = \Pr(\return > \returnsample)$. Since the RWs are independent, we can assume that a node $\nodeidx$ evaluates the survival function of RW $\rwidxtwo$ at a random point $t$ in time. This point in time is when a random walk $\rwidx \neq \rwidxtwo$ visits node $\nodeidx$, which is random and independent of $\rwidxtwo$. Hence, the observed $\returnsample$ is a random sample itself following the distribution of $\return$.
Consider now $\nactive$ RWs that have been active for a long time such that each has visited each node $\nodeidx \in [\nnodes]$ at least once. The expectation of the estimation $\rwestri$ over the randomness of $\returnsample \sim \return$ at time $t$ for a single active random walk is $\E\left[ \rwestri \right] = \frac{1}{2}$. This follows from the probability integral transform \cite{david1948probability}, which states that the CDF of any random variable with invertible CDF evaluated at the random variable itself is uniformly distributed on $\mathcal{U}(0,1)$. By a symmetry argument, this equivalently holds for the survival function. %
\end{proof}

\cref{prop:unbiased_estimator} justifies the use of the offset $\frac{1}{2}$ in \algname and \algnameplus for the visiting RW $\rwidx$, which is known to be active. Similarly, for the remaining $\nactive-1$ RWs, we have $\sum_{\ell\in \lastseenset\setminus\{\rwidx\}}\E\left[ \rwestri \right] = (\nactive-1)/2$. Note that \cref{prop:unbiased_estimator} holds under \cref{assumption:distributions} and for any continuous return time distribution with invertible CDF. Hence, we investigate the error made by the continuous approximation of the return time distributions using the PMF and CDF of the geometric distribution supported on $\returnsample \in \{1, 2, \dots, \infty\}$ with parameter $q$ and probability mass function $\Pr(\return = \returnsample) = (1-q)^{\returnsample-1} q$. For an active random walk $\activeidx$, the expected value of $\E[\rwestri]$ for $\return$ distributed according to a geometric distribution with parameter $q$ %
reads as \vspace{-.2cm}
\begin{align}
    \E\left[ \rwestri \right] &\!= \!\!\! \sum_{\returnsample = 1}^\infty \!\! \Pr(\return = \returnsample) \rwestri \nonumber %
    = \!\!\! \sum_{\returnsample = 1}^\infty \! (1-q)^{2\returnsample-1} q = \! \frac{1-q}{2-q}. %
\end{align}
Hence, for small values of $q$, such as those expected for random regular graphs \cite{tishby2021analytical}, the expectation is close to $0.5$ (as for the analytical counterpart), but a non-zero error remains. While this did not affect the algorithm's performance negatively in our experiments, the constant offset $\frac{1}{2}$ can be replaced by the actual expectation using the empirical distribution of $\return$ established at all the nodes.

\subsection{On the Distribution of the Estimator $\nest$}
We study the forking probability based on $\nest$ to provide worst-case guarantees on the reaction time to failure events and to bound the undesired increase of $\nwalks$ beyond $\ntarget$. To that end, instead of relying on concentration bounds of $\nest$ around its average, we derive its probability distribution in the sequel. For any RW $\rwidxtwo$ forked by a random node at time $\forktime$ and terminated at time $\termtime$, we derive in \cref{lemma:forkedcdf} the exact distribution of the estimator $\forkestrv \define \rwest[\rwidxtwo]$ evaluated at any point in time $t$ by a random node $\nodeidx$. The distribution is the building block for the theoretical guarantees of our algorithms.

\begin{lemma} \label{lemma:forkedcdf}
    For an RW $\rwidx$ forked by a random node $\nodeidx$ at time $\forktime<t$, the distribution of $\rwest$ is given by the CDF \vspace{-.5cm}
    \begin{align*}
     \forkcdf[\forkcdfvar] = \begin{cases}
        1, & \hspace{-1.6cm} e^{-\returnrate (t-\termtime)} < \forkcdfvar \\
        e^{-\arrivalrate (\termtime-\forktime)}, & \hspace{-1.6cm}\forkcdfvar < e^{-\returnrate (t-\forktime)} \\
        \frac{\forkcdfvar\left(1 - e^{-\arrivalrate (t - \forktime)} x^{-\frac{\arrivalrate}{\returnrate}} \right)}{e^{-\returnrate (t-\termtime)}} \! + \! e^{-\arrivalrate (\termtime-\forktime)}, & \text{else}. %
    \end{cases}
    \end{align*}
    When $\forktime < t < \termtime$, i.e., the RW has not (yet) been terminated, the distribution $\forkcdf[\forkcdfvar]$ is given by setting $\termtime = t$.
\end{lemma}

\begin{corollary} \label{cor:avg_est}
    The average of $\forkestrv$ is given by \vspace{-.2cm}
    \begin{align*}
    \mathbb{E}&[\forkestrv] = e^{-\arrivalrate (\termtime - \forktime)} e^{-\returnrate (t-\termtime)} \left(\frac{1}{2-\frac{\arrivalrate}{\returnrate}}-1 \right) +  \\
    & + \frac{e^{-\returnrate (t-\termtime)}}{2} + e^{-2\returnrate (t-\forktime)} e^{\returnrate (t-\termtime)} \bigg(\frac{1}{2}-\frac{1}{2-\frac{\arrivalrate}{\returnrate}} \bigg).
    \end{align*}
\end{corollary}

Having established the distribution and the average of $\forkestrv$, we can prove that the estimate $\nest$ is unbiased, asymptotically in the time $t$ for any node $\nodeidx$. For $\termtime=t$ and $\forktime=-\infty$, i.e., the RW is active for infinitely long, we can show that $\mathbb{E}[\forkestrv] = \frac{1}{2}$, which is consistent with \cref{prop:unbiased_estimator}. 
\cref{thm:unbiasedness} shows that the estimator is tracking the actual number of random walks in the system for arbitrary forks and terminations.

\begin{theorem} \label{thm:unbiasedness}
    The estimator $\nest$ resembles asymptotically (over time) half the actual number of active random walks $\frac{\nwalksrw}{2}$. In particular, with $\lastevent$ the time of the last event, i.e., either the termination or the fork of any RW, for any node $\nodeidx$, we have
    \begin{equation*}
        \lim_{t-\lastevent \rightarrow \infty} \mathbb{E}[\nest] = \nactive,
    \end{equation*}
    where $\nactive$ is the number of active RWs between $\lastevent$ and $t$.
\end{theorem}
\begin{proof}
    First note the following result, which follows immediately from \cref{cor:avg_est}, the definition of the estimator $\nest$ in \eqref{eq:estimator} and from the linearity of expectation.
    \begin{proposition} \label{prop:average_estimation}
    Let $\forksym_{\forktime, \termtime}$ be the (potentially empty) set of all RWs created/forked at time $\forktime$ and terminated at time $\termtime$, where $\termtime=t$ for active RWs. At time $t$, we have
    \begin{align*}
        \mathbb{E}&[\nest] = \frac{1}{2} + \sum_{\forktime, \termtime \leq t} \vert \forksym_{\forktime, \termtime} \vert  \mathbb{E}[\forkestrv],
    \end{align*}
    where $\mathbb{E}[\forkestrv]$ is given by \cref{cor:avg_est}.
    \end{proposition}
    We study $\mathbb{E}[\forkestrv]$ for all possible types of events. From \cref{cor:avg_est}, it follows that the estimation is unbiased for all RWs active infinitely long. For RWs created at $\forktime=-\infty$ (and hence active infinitely long) and terminated at time $\termtime$, the estimation $\mathbb{E}[\forkestrv] = \frac{e^{-\returnrate(t-\termtime)}}{2}$ converges to $0$ as $t-\termtime \rightarrow \infty$. For active random walks forked at time $\forktime$, we set $\termtime=t$. The estimation $\mathbb{E}[\forkestrv]$ goes to $\frac{1}{2}$, as $\lim_{t-\forktime \rightarrow \infty} e^{-\arrivalrate (t - \forktime)} = 0$ and $\lim_{t-\forktime \rightarrow \infty} e^{-2\returnrate (t - \forktime)} = 0$, $e^{-\returnrate (t-\termtime)} = 1 = \text{const.}$, and $\frac{e^{-\returnrate(t-\termtime)}}{2}=\frac{1}{2}$. For random walks forked at time $\forktime$ and terminated at time $\termtime > \forktime$, we have $e^{-\arrivalrate (\termtime - \forktime)} = \text{const.}$, $\lim_{t-\termtime \rightarrow \infty} e^{-\returnrate (t-\termtime)} = e^{-2\returnrate (t-\forktime)} = 0$. Hence, the estimated value $\mathbb{E}[\forkestrv]$ converges to $0$. Let $\lastevent$ be the time of the last event, i.e., either the termination of the fork of any random walk, then the above considerations hold for $t-\lastevent \rightarrow \infty$.  We conclude by using \cref{prop:average_estimation} and plugging $\mathbb{E}[\forkestrv]$ for any possible of the above events.
\end{proof}

We have shown that the average of the estimator resembles the actual number of RWs when waiting sufficiently long after an event, either a fork or a termination, happens. For ease of presentation and to better observe the exact qualitative behavior, we focus on three types of RWs. RWs that are active at time $t$ since infinitely long, RWs that were active for infinitely long and terminated at time $\termtime$, and RWs that were forked at time $\forktime$. %
Hence, we do not consider RWs forked at time $\forktime$ and later terminated at time $\termtime > \forktime$. Analyzing the resulting properties of the estimator will give valuable and tractable insights into its behavior and allow for a detailed analysis of the algorithms.

For the remainder of the paper, let $\activesym_t$ be the set of all random walks at time $t$ active since infinitely long, %
$\terminatedsym_{\termtime}$ the set of random walks terminated at time $\termtime \leq t$, and $\forksym_{\forktime}$ the random walks created through forks at time $\forktime$. We let $\termtimes$ and $\forktimes$ be sets containing the time instances at which at least one termination or fork happened, respectively. We define $\termtimes^{\leq t} \define \termtimes \cap [t]$, and similarly $\forktimes^{\leq t} \define \forktimes \cap [t]$. Denote by $\terminatedsym^{\leq t} \define \cup_{t^\prime=1}^t \terminatedsym_{t^\prime} = \cap_{t^\prime \in \termtimes^{\leq t}} \terminatedsym_{t^\prime}$ the set of all terminated random walks at time $t$, and by $\forksym^{\leq t} \define \cup_{t^\prime=1}^t \forksym_{t^\prime} = \cup_{t^\prime \in \forktimes^{\leq t}} \forksym_{t^\prime}$. $\terminatedsym_{t^\prime}$ and $\forksym_{t^\prime}$ are empty when at $t^\prime$ no termination or fork happened, respectively. Note all sets are disjoint by the assumption above. We start with deriving the properties of $\nest$ when restricting to the above cases.

\begin{observation} \label{obs:estimator}
    Disregarding the case that forked RWs can later terminate, the estimator $\nest$ for a node $\nodeidx$ visited by random walk $\rwidx \in \activesym_t$ at time $t$ is composed of
    \begin{align*}
        &\nest = \frac{1}{2} + \overbrace{\sum_{\ell\in \lastseenset \cap \activesym_t \setminus\{\rwidx\}} \!\!\!\!\! \rwest[\rwidxtwo] }^{\text{Initial random walks}} + \\
        &+ \!\!\!\!\! \underbrace{\sum_{\termtime \in \termtimes^{\leq t}} \sum_{\ell\in \lastseenset \cap \terminatedsym_{\termtime}} \!\!\!\!\!\!\!\!\!\! \rwest[\rwidxtwo]}_{\text{Terminated initial random walks}} 
        + \!\!\!\!\!\! \underbrace{\sum_{\forktime \in \forktimes^{\leq t}} \sum_{\ell\in \lastseenset \cap \forksym_{\forktime}} \!\!\!\!\!\!\!\!\!\!\!\!\!\!\! \rwest[\rwidxtwo].}_{\text{Forked and active random walks}} %
    \end{align*}
\end{observation}
\begin{lemma} \label{prop:average_estimation}
    For a node $\nodeidx$ visited by an RW $\rwidx \in \activesym_t$ at time $t$,
    \begin{align*}
        \mathbb{E}&[\nest] = \frac{1}{2} + \frac{\vert \activesym_t \vert - 1}{2}  + \sum_{\termtime \in \termtimes^{\leq t}} \vert \terminatedsym_{\termtime} \vert \frac{e^{-\returnrate(t-\termtime)}}{2} \\[-.3cm]
        &+ \sum_{\forktime \in \forktimes^{\leq t}} \vert \forksym_{\forktime} \vert  \left(\frac{1}{2} + e^{-\arrivalrate (t - \forktime)} \bigg(\frac{1}{2-\frac{\arrivalrate}{\returnrate}}-1\bigg) + \right. \\[-.3cm]
        &\left. + e^{-2\returnrate (t-\forktime)} \bigg(\frac{1}{2}-\frac{1}{2-\frac{\arrivalrate}{\returnrate}} \bigg) \right).
    \end{align*}
\end{lemma}
\begin{proof}
    Recall that both active and terminated RWs were assumed active for an infinite time period. We start from \cref{obs:estimator}. By the assumption that the graph is connected, it holds with probability $1$ that $\activesym_t \subseteq \lastseenset$ and $\forall \termtime \in \termtimes^{\leq t}: \terminatedsym_{\termtime} \subseteq \lastseenset$, i.e., all active and terminated RWs are known to node $\nodeidx$. Hence, $\lastseenset \cap \activesym_{t} = \activesym_{t}$ and $\lastseenset \cap \terminatedsym_{\termtime} = \terminatedsym_{\termtime}$. %
    By the linearity of the expectation, we analyze each term in \cref{obs:estimator} individually. For any active RW $\rwidxtwo \in \activesym_t$ the following statement follows from the probability integral transform \cite{david1948probability} and the symmetry between the CDF and the survival function.
    \begin{observation} \label{obs:active}
    For an active RW $\rwidxtwo$, $\rwest[\rwidxtwo]$ is uniformly distributed between $0$ and $1$, i.e., $\rwest[\rwidxtwo] \sim \uniform(0,1)$ for all $\rwidxtwo \in \activesym_t$.
    \end{observation}
    Hence, we have $\mathbb{E}[\rwest[\rwidxtwo]] = \frac{1}{2}$ for all $\rwidxtwo \in \activesym_t$, and thus $\mathbb{E}\left[\sum_{\ell\in \lastseenset \cap \activesym_t \setminus\{\rwidx\}} \rwest[\rwidxtwo]\right] = \sum_{\ell\in \activesym_t \setminus \{\rwidx\}} \mathbb{E}\left[\rwest[\rwidxtwo]\right] = \frac{\vert \activesym_t \vert - 1}{2}$. For any terminated RW $\rwidxtwo \in \terminatedsym_{\termtime}, \termtime \in \termtimes^{\leq t}$, we have the following observation for the individual estimations $\rwest$, which again follows from the same arguments as for \cref{obs:active} and the fact that values of $t-\lastseen[\rwidxtwo]<t-\termtime$ cannot be observed.
    \begin{observation} \label{obs:term}
    For an RW $\rwidxtwo$ terminated at $\termtime$, for $t>\termtime$, $\rwest[\rwidxtwo]$ is distributed uniformly between $0$ and $e^{-(t-\termtime)}$, i.e., $\rwest[\rwidxtwo] \sim \uniform(0, e^{-(t-\termtime)})$.
    \end{observation}
    Hence, we have $\forall \rwidxtwo \in \terminatedsym_{\termtime}, \termtime \in \termtimes^{\leq t}: \mathbb{E}[\rwest[\rwidxtwo]] = \frac{e^{-(t-\termtime)}}{2}$, and thus $\mathbb{E}\left[\sum_{\ell\in \lastseenset \cap \terminatedsym_{\termtime} \setminus \{\rwidx\}} \rwest[\rwidxtwo]\right] = \sum_{\ell\in \terminatedsym_{\termtime}} \mathbb{E}\left[\rwest[\rwidxtwo]\right] = \vert \terminatedsym_{\termtime} \vert \frac{e^{-(t-\termtime)}}{2}$. The third term follows by setting in \cref{cor:avg_est} the termination time $\termtime=t$ and noting that $\rwidx \notin \cup_{\termtime \in \termtimes^{\leq t}} \terminatedsym_{\termtime}$. Hence, we sum over $\sum_{\forktime \in \forktimes^{\leq t}} \vert \forksym_{\forktime} \vert$ independent instances of terminated RWs, each with mean given by \cref{cor:avg_est}, where $\termtime=t$.
\end{proof}

To establish a rule for the design of the threshold $\thres$ of the algorithm and in preparation for further theoretical results, we analyze the distributions of the terms in \cref{obs:estimator} corresponding to the active and the terminated RWs. We state the results in \cref{prop:dist_active,prop:dist_failed}.
\begin{proposition} \label{prop:dist_active}
    Under \cref{assumption:distributions} and assuming that $F_{\return}(t)$ is continuous and invertible, for $\stablenrws$ active RWs $\activesym_t$ in the system, the estimation $\nest$ is a random variable that can be described by the CDF $F_{\Sigma_{\stablenrws-1}}(\sigma)$, where \vspace{-.2cm}
    \begin{align*}
    F_{\Sigma_{\stablenrws-1}}(\sigma) = \frac{1}{(\stablenrws-1)!} \sum_{\tau = 0}^{\lfloor \sigma \rfloor} (-1)^\tau \binom{\stablenrws-1}{\tau} (\sigma-\tau)^{\stablenrws-1}. \vspace{-.2cm}
\end{align*}
\end{proposition}
\begin{proof}[Proof of \cref{prop:dist_active}]
    The statement follows from \cref{obs:active} and the definition of an Irwin-Hall distribution as the sum of independently and identically distributed uniform random variables $\uniform(0,1$).
\end{proof}
This result can be used to design the thresholds $\thres$ and $\thresterm$ as described in \cref{sec:main_results}. 
To bound the reaction time of the algorithm and the overshoot, we consider a single failure event that occurs at time $\termtime$, leading to the failure of $\nterm$ RWs.
\begin{proposition} \label{prop:dist_failed}
    Under \cref{assumption:distributions} and assuming that $F_{\return}(t)$ is continuous and invertible, for $\nterm$ random walks indexed by $\terminatedsym_{\termtime}$ terminated at time $\termtime$, the part of $\nest$ corresponding to $\terminatedsym_{\termtime} \subset \lastseenset$ can be described by the CDF $F_{\Sigma_{\nterm}} (\sigma e^{\returnrate(t-\termtime)})$.
\end{proposition}
\begin{proof}[Proof of \cref{prop:dist_failed}]
    From \cref{obs:term}, the distributions for all RWs in $\terminatedsym_{\termtime}$ are equal and given by $\uniform(0, e^{-(t-\termtime)})$. The distribution of the sum of the estimations $\rwest$, which corresponds to the part in $\nest$ indexed by the failed RWs in $\terminatedsym_\termtime \subseteq \lastseenset$, is given by a scaled Irwin-Hall distribution corrected by the support of the uniform distributions.
\end{proof}

\subsection{A Bound on the Reaction Time to Failure Events}

We derive worst-case guarantees on the reaction time as a response to the failure of $\nterm$ random walks at time $\termtime$. In Section~\ref{sec:D}, we show that improving the reaction time increases the probability of reaching beyond $\ntarget$ active RWs. 
We assume that $\nactive^\prime$ RWs have been active for long enough to visit all nodes at least once. After $\nterm$ RWs fail at time $\termtime$, $\nactive = \nactive^\prime - \nterm$ RWs remain active.

We bound the time $\tforkn$ elapsed until at least $\recovered^\prime \leq \nterm$ RWs are forked with a certain probability. The main ingredient is to bound the time $\tlbfork[\nterm-\recovered]$ elapsed until at least one node forks an RW after $\nterm$ RWs failed and $\recovered$ forks took place, with probability $1-\deltanoforksol[\nterm-\recovered]$, for some $0 \leq \recovered < \nterm$ and $0<\deltanoforksol[\nterm-\recovered]<1$. This result is given in \cref{prop:lower_bound}.

\begin{theorem} \label{prop:lower_bound}
    Consider the setting explained above and the event where $\nterm$ RWs fail and $\recovered$ forks happen afterward. For any choice of $ 0 < \epstmp < \thres-\frac{1}{2}$ and $\ttotal>0$, let the quantity $\deltanofork[\nterm-\recovered]$ be bounded by
    \begin{align*}
        \deltanofork[\!\nterm-\!\recovered] & \!\! \leq \!\!\! \prod_{t = \termtime}^\ttotal \! \left[\!1 \!\! - \!\pforkscale F_{\Sigma_{\nactive+\recovered-1}}(\epstmp) F_{\Sigma_{\nterm-\recovered}}\!\left(\!\frac{\thres-\epstmp-\frac{1}{2}}{e^{-\returnrate (t-\termtime)}}\!\right)\! \right]\!.
    \end{align*}
    
    For a desired $\deltanoforksol[\nterm-\recovered]>0$, the time $\tlbfork[\nterm-\recovered]$ until at least one fork occurs with probability at least $1-\deltanoforksol[\nterm-\recovered]$, is bounded by the smallest $\ttotal$ satisfying $1-\deltanofork[\nterm-\recovered] \geq 1-\deltanoforksol[\nterm-\recovered]$. 
\end{theorem}
The $\epstmp$ can be chosen to minimize $\tlbfork[\nterm-\recovered]$. 
Applying \cref{prop:lower_bound} for $\recovered \in \{0, \dots, \recovered^\prime\!\!-\!\!1\}$, with $\deltasum \define \sum_{\recovered=1}^{\recovered^\prime-1} \deltanoforksol[\nterm-\recovered]$, we can write $\Pr\left(\tforkn \leq \sum_{\recovered=0}^{\recovered^\prime-1} \tlbfork[\nterm-\recovered]\right)\geq 1-\deltasum$. 
The parameter $\deltasum$ can be split into the $\deltanoforksol[\nterm-\recovered]$'s to minimize $\tforkn$.

An implication of~\cref{prop:lower_bound} is that the time to fork increases with the number of forked RWs. %

\subsection{The Number of Random Walks is Finite}\label{sec:D}
To bound the maximum number of RWs in the system, we bound the maximum number of forks that occur when using \algname for a duration $\ttotal$ without any failure.
Assume a time $t$ at which $\nwalks = \rwstep$. The probability of forking is bounded by $\pforkcurrent[\rwstep] \leq \pforkcurrent[\rwstep]^+ \define \rwstep \cdot \pforkscale \cdot F_{\Sigma_{\rwstep-1}}(\thres-\frac{1}{2})$, where the factor $\rwstep$ results from at most $\rwstep$ distinct nodes being visited by an RW.
For $\thres < 1$, the forking probability simplifies to $\pforkcurrent[\rwstep] \leq \pforkcurrent[\rwstep]^+= \frac{\rwstep \pforkscale (\thres-\frac{1}{2})^{\rwstep-1}}{(\rwstep-1)!}.$ 
For $\rwstep+1$, we have $\pforkcurrent[\rwstep+1] \leq \frac{(\rwstep+1)\pforkscale (\thres-\frac{1}{2})^{\rwstep}}{\rwstep!} =  \pforkcurrent[\rwstep]^+ \frac{(\rwstep+1)(\thres-\frac{1}{2})}{\rwstep^2}$. 
Hence, for any RW that gets forked in a system of $\rwstep$ RWs, subsequent forking probabilities decrease by a factor of $\frac{(\rwstep+1)(\thres-\frac{1}{2})}{\rwstep^2}$ in the long run. However, the forking probability decreases only when all nodes are aware of all active RWs in the system. This intuition is the basis of the proof of the following theorem, which bounds the probability of $\nwalks$ exceeding the number $\nbound>\ntarget$ in a graph operating for a duration $\ttotal$ without failures.

\begin{theorem} \label{thm:upper_bound}
    For $\forkidxtmp<\nbound$, let $\tallvisiti =\frac{1}{\arrivalrate} \log(\frac{\arrivalrate \nnodes}{\pforkcurrent[\rwstep]^+})$. After time $\ttotal$, the probability of having more than $\nbound>\ntarget$ walks in the network is bounded as follows, for some $\nmax \leq \nbound$, \vspace{-.1cm} %
    \begin{align*}
        \delta %
        &\leq \pforkcurrent[\nmax]^+ \tnoforki[\nmax] + \sum_{\forkidxtmp=\ntarget}^{\nmax-1} \nnodes e^{-\arrivalrate \tallvisiti} + \tallvisiti \pforkcurrent[\rwstep]^+.  \vspace{-.1cm}
    \end{align*}
     The statement holds for $\nmax$ being the largest integer (smaller than $\nbound$) so that $\sum_{\forkidxtmp=\ntarget}^{\nmax-1} \tallvisiti < \ttotal$. The time $\tnoforki[\nmax]$ must then be chosen as $\tnoforki[\nmax] = \ttotal - \sum_{\forkidxtmp=\ntarget}^{\nmax-1} \tallvisiti$.
\end{theorem}
\cref{thm:upper_bound} can be inverted to state for any confidence $\delta>0$, the probability $\Pr(\nwalks<z) \geq 1 - \delta$ as long as the algorithm runs for a time $\ttotal$ bounded as in \cref{cor:delta}.
\begin{corollary}\label{cor:delta}
    With probability at most $\delta$, the time $\ttotal$ until the number of RWs grows larger than $\nbound$ is bounded by \vspace{-.1cm}
    \begin{align*}
        \ttotal \geq \tnoforki[\nmax] + \sum_{i=\ntarget}^{\nmax-1} \tallvisiti,
    \end{align*}
    where $\tallvisiti$ is as above and $\nmax$ is the largest integer such that $\delta < \delta_\Sigma \define \sum_{i=\ntarget}^{\nmax-1} \nnodes e^{-\arrivalrate \tallvisiti} + \tallvisiti \pforkcurrent[\rwstep]^+$, and $\tnoforki[\nmax] = \frac{\delta - \delta_\Sigma}{\pforkcurrent[\nmax]^+}$.  \vspace{-.1cm}
\end{corollary}
The trade-off between reaction time and the likelihood of increasing beyond $\nbound$ RWs after the start of \algname is controlled by the choice of $\thres$ and is implicit in \cref{prop:lower_bound,thm:upper_bound}. The smaller $\thres$, the larger the times $\tallvisiti$ in \cref{thm:upper_bound}, which reflect fewer undesired forks at a given time. Conversely, smaller values for $\thres$ lead to smaller values of the CDFs in \cref{prop:lower_bound}, hence to a slower decrease of the product and, thus, a larger delay to fork. This trade-off aligns with numerical experiments for different values of $\thres$, which we depict in \cref{fig:rrg_different_eps}. The larger $\thres$, the larger the average number of RWs in the system, but the faster the reaction time. Choosing even smaller values for $\thres$ will likely lead to failures of the system after the second burst failure at time $t=6000$. %

\subsection{Bounding the Overshoot after Failures}

One of the main challenges is to bound the overshoot of the number of RWs after the burst failures of multiple RWs. This is due to rapidly increasing forking probabilities at the nodes, potentially leading to over-forking when new forks are not immediately detected by the non-forking nodes. Deriving such a result relies on an upper bound on the forking probability at any point $t$ in time given the history of all forks and terminations, denoted by $\history \define \{\terminatedsym_{\termtime}, \forksym_{\forktime}: \termtime \in \termtimes^{\leq t}, \forktime\in \forktimes^{\leq t}\}$. Such an upper bound on the forking probability is given in the sequel, which relies on the variance of the estimator $\forkestrv$ for any RW:

\begin{lemma} \label{lemma:variance}
Let $\arrivalrate \neq 2 \returnrate$ and $\arrivalrate \neq 3 \returnrate$, we have
\begin{align*}
    &\mathrm{Var}[\forkestrv] = \frac{e^{\termtime(\returnrate-\arrivalrate) - 4\returnrate t}}{12 (\arrivalrate - 3 \returnrate) (\arrivalrate - 2 \returnrate)^2}
\Big( 
    3 (-\arrivalrate + 3 \returnrate) \cdot \\
    &
    \cdot \left( 2e^{\arrivalrate (\forktime - \termtime)}(\returnrate-\arrivalrate) + 
        \arrivalrate e^{2 \returnrate \cdot (\forktime - \termtime)} + 
        \arrivalrate -  2 \returnrate 
    \right)^2 \cdot \\
    & \cdot e^{(\arrivalrate+\returnrate) \termtime + 2 \returnrate t} \! + \! 
    4 (\arrivalrate \! - \! 2 \returnrate)^2 \! \cdot \! e^{2 \returnrate (t - \termtime)} \big(2 \arrivalrate e^{\arrivalrate \termtime + 3 \forktime \returnrate}  \\
    &
        + (\arrivalrate-3\returnrate) e^{\termtime (\arrivalrate + 3 \returnrate)} - 
        (\returnrate-\arrivalrate) 3 e^{\arrivalrate \forktime + 3 \returnrate \termtime}
    \big) \Big)
\end{align*}
The proof is technical and omitted for brevity.
\end{lemma}

At any point $t$ in time, the forking probability is bounded from above as given in the following theorem.
\begin{lemma} \label{thm:fork_bound}
    Let $h(\zeta) = (1 + \zeta) \log(1 + \zeta) - \zeta$. Consider an RW $\rwidx$ visiting node $\nodeidx$. For $\mathbb{E}[\nest] > \thres$, the forking probability is bounded by \vspace{-.5cm}
    \begin{align*}
        \pforkt \leq \forkboundt \define \! \pforkscale \exp\left( \!\! - \sigma^2(t) h\left(\!\! \frac{\left(\mathbb{E}[\nest] - \thres\right)^2}{\sigma^2(t)}\right) \!\! \right),
    \end{align*}
    where $\mathbb{E}[\nest]$ is given by \cref{prop:average_estimation} and $\sigma^2(t) = \frac{\vert \activesym_t \vert-1}{12} + \sum_{\forktime \in \forktimes^{\leq t}} \vert \forksym_{\forktime} \vert \mathrm{Var}[\forkestrv] + \sum_{\termtime \in \termtimes^{\leq t}} \vert \terminatedsym_{\termtime} \vert \frac{e^{-2\returnrate (t-\termtime)}}{12}$. \vspace{-.2cm} %
\end{lemma}
\begin{proof}
    The proof is by applying Bennett's concentration bound for the sum of random variables bounded between $0$ and $1$, by the variance of uniform random variables and considering that forks happen with probability $\pforkscale$ when $\nest < \thres$.
\end{proof}

\cref{thm:overshoot}, given below, provides an analytical handle on the expected number of RWs long after the failure event happened; hence bounds the overshoot of the algorithm. 
Since we assume in the following result that no terminations happen after $\termtime$, the number of RWs at each time step is a full indicator of the history, i.e., how many forks happened each time step. Hence, the history is given by $\history = \{\nwalks[t-1], \cdots, \nwalks[0]\}$. We denote in the following vectors by bold letters, e.g., $\mathbf{a}$, and the $x$-th entry by $\mathbf{a}[x]$. The first $x$ entries are denoted by $\mathbf{a}[:x]$, and $1, \mathbf{a}$ denotes the concatenation.

\begin{theorem} \label{thm:overshoot}
Let $\{\nrwtmp_{1,\mathbf{a}}\}_{\mathbf{a}\in \{0,1\}^{\leq x-2}}$ be an appropriate choice of thresholds and $\{0,1\}^{\leq x}$, $x>0$, be the set of all binary vectors of length at most $x$. Assuming no forks until any time $\firstfork$, and $\nwalks[\firstfork]$ remaining RWs after failures at time $\termtime < \firstfork$, the expected number of RWs at time $t + \firstfork >> \termtime$ is bounded by
\begin{align*}
    &\mathbb{E}[\nwalksrw[\firstfork+x]] \leq \!\!\!\!\!\! \sum_{\mathbf{a} \in \{0,1\}^{x-1}} \!\!\!\!\!\!\left(\nwalks[\firstfork+x-1] + \nwalks[\firstfork+x-1] \forkboundt[\firstfork+x-1]\right) \\[-.2cm]
    &\! \prod_{x^\prime=1}^{x-1} \!\! \mathds{1}\{\mathbf{a}[x^\prime]\! =\! 1\} \! \Pr(\nwalksrw[\firstfork+x^\prime] \! > \! \nrwtmp_{1,\mathbf{a}[:x^\prime-1]} \vert \nwalks[\firstfork+x^\prime-1], \history[\firstfork+x^\prime-1])),
\end{align*}
where $\Pr(\nwalksrw[\firstfork+x^\prime] \! > \! \nrwtmp_{1,\mathbf{a}[:x^\prime-1]} \vert \nwalks[\firstfork+x^\prime-1], \history[\firstfork+x^\prime-1]))$ \vspace{-.2cm}
\begin{align*}
    &\leq \!\!\!\!\! \sum_{\nrwtmp = \nrwtmp_{1,\mathbf{a}[:x^\prime-1]}+1}^{2 \nwalks[\firstfork+x^\prime-1]} \!\!\!\! \binom{\nwalks[\firstfork+x^\prime-1]}{\nrwtmp -\nwalks[\firstfork+x^\prime-1]} \begin{aligned}[t] & \forkboundt[\firstfork+x^\prime-1]^{\nrwtmp-\nwalks[\firstfork+x^\prime-1]} \\
    &\!\!\!\!\!\!\!\!\!\!\!\!\!\cdot \left(1-\forkboundt[\firstfork+x^\prime-1]\right)^{2\nwalks[\firstfork+x^\prime-1] - \nrwtmp}\!. \end{aligned}
\end{align*}
The history $\history[\firstfork+x] = \{\nwalks[\firstfork+x-1], \history[\firstfork+x-1]\}$ is a function of $\mathbf{a}$ and given as $\forall x^\prime \in [2, x]: \nwalks[\firstfork+x^\prime-1] = \left(\mathds{1}\{\mathbf{a}[x^\prime-1]=0\} \nrwtmp_{1,\mathbf{a}[:x^\prime-2]}+\mathds{1}\{\mathbf{a}[x^\prime-1]=1\} 2 \nwalks[\firstfork+x^\prime-2] \right)$. %

\end{theorem}
For a fixed $t$, the worst case is given for large $\firstfork$, which maximizes the forking probabilities, thereby maximizing the likelihood of many undesired forks. Further, the forking probability tends to $0$ in the case of many past forks, and hence, this bound is finite when $t \rightarrow \infty$. The thresholds $\{\nrwtmp_{1,\mathbf{a}}\}_{\mathbf{a}\in \{0,1\}^{\leq x-1}}$ must satisfy $\nrwtmp_{1, \mathbf{a}[:x^\prime], 1} > \nrwtmp_{1, \mathbf{a}[:x^\prime]}$ and $\nrwtmp_{1, \mathbf{a}[:x^\prime], 0} \leq 2\nrwtmp_{1, \mathbf{a}[:x^\prime]}$ for all $\mathbf{a} \in \{0,1\}^{\leq x-2}$ and $0 \leq x^\prime < x-2$ and can be optimized to minimize the bound.

Since the exact bound in \cref{thm:overshoot} is difficult to evaluate due to its exponential complexity in $t-\firstfork$, we provide in the following an approximate bound with linear complexity. The bound follows from \cref{thm:overshoot} by assuming that at each iteration the expected number of forks has happened.
\begin{corollary}
    The expected number of RWs  $\forall t^\prime \in [\termtime+1,t]$ is approximately bounded by
    \begin{align*}
        \mathbb{E}[\nwalksrw[t^\prime]] \lesssim \bar{\mathbb{E}}[\nwalksrw[t^\prime]] \define \lceil \mathbb{E}[\nwalksrw[t^\prime-1]] \rceil + \lceil \mathbb{E}[\nwalksrw[t^\prime-1]] \rceil \forkboundt[t^\prime-1],
    \end{align*}
    where $\history[t^\prime-1] = \{\lceil \bar{\mathbb{E}}[\nwalksrw[t^\prime-2]] \rceil, \cdots, \lceil \bar{\mathbb{E}}[\nwalksrw[\termtime+1]]\rceil, \nwalks[\termtime] \}$. The recursive application gives a bound for $\mathbb{E}[\nwalksrw[t]]$.
\end{corollary}
While this provides an indicator for the expected growth of the number of RWs, it does not converge, since $\nwalks$ is non-decreasing (hence, a submartingale), and $\forkboundt[t^\prime] \neq 0$ by the exponential assumption. Consequently, the ceiling operator leads to an increment of at least $1$ from $\bar{\mathbb{E}}[\nwalksrw[t^\prime-1]]$ to $\bar{\mathbb{E}}[\nwalksrw[t^\prime]]$.

\begin{figure}[tb]
    \centering
    \resizebox{0.8\linewidth}{!}{\input{plots/tuned_eps_10_different_nodes}}
    \vspace{-0.1cm}
    \caption{\small Consistent performance of \algname for random $8$-degree regular graphs with different numbers of nodes $\nnodes \in \{50, 100, 200\}$ and $Z_0=10$. %
\vspace{-0.5cm}}
    \label{fig:rrg_different_n}
\end{figure}

\section{Bounding the Terminations in \algnameplus} \label{sec:theory_algplus}

\algnameplus additionally allows deliberate terminations to the algorithm. To limit the termination probability despite sufficiently few RWs being active, we establish \cref{thm:term_bound}.
\begin{lemma} \label{thm:term_bound}
    Let $h(\zeta) = (1 + \zeta) \log(1 + \zeta) - \zeta$. Consider an RW $\rwidx$ visiting node $\nodeidx$. For $\mathbb{E}[\nest] < \thresterm$, the termination probability is bounded by \vspace{-.4cm}
    \begin{align*}
        \ptermt \leq \pforkscale \exp\left(\!\! - \sigma^2(t) h\! \left(\! \frac{\left(\thresterm-\mathbb{E}[\nest]\right)^2}{\sigma^2(t)}\! \right)\! \right), %
    \end{align*}
    where $\mathbb{E}[\nest]$ is given by \cref{prop:average_estimation} and $\sigma^2(t) = \frac{\vert \activesym_t \vert-1}{12} + \sum_{\forktime \in \forktimes^{\leq t}} \vert \forksym_{\forktime} \vert \mathrm{Var}[\forkestrv] + \sum_{\termtime \in \termtimes^{\leq t}} \vert \terminatedsym_{\termtime} \vert \frac{e^{-2\returnrate (t-\termtime)}}{12}$. %
\end{lemma}
\begin{proof}
    The proof is analog to the proof of \cref{thm:fork_bound}.
\end{proof}
This result, combined with the upper bound on the forking probability in \cref{thm:fork_bound}, determines an important property of \algnameplus: if $\mathbb{E}[\nest] \gg \thres$, the probability of forking is negligible, and if $\mathbb{E}[\nest] \ll \thresterm$, the termination probability is negligible. Since $\mathbb{E}[\nest]$ tracks the actual number of RWs according to \cref{thm:unbiasedness}, \algnameplus balances forks and terminations so that $\nwalks$ remains in a corridor around $\ntarget$ with high probability. This is consistent with our simulation results.

\section{Numerical Experiments} \label{sec:add_experiments}

We provide further experiments showing the performance of \algname for different numbers of nodes, types of graphs, and the trade-off between reaction time and overshooting.

To show the robustness of \algname with a varying number of $\nnodes$ of nodes, the performance on a random regular graph is depicted in \cref{fig:rrg_different_n}. For $50$ numerical simulations, we generated $8$-degree regular graphs, for $\nnodes \in \{50, 100, 200\}$ and $\ntarget=10$ desired RWs on the graph. At time $t=2000$ and $t=6000$, we impose failure events that deterministically result in the failure of $5$ and $6$ RWs, respectively. The value of $\thres \in \{1.85,2,2.1\}$ is well-tuned for the respective number $\nnodes$. We observe the desired behavior of forking RWs after the failure events. Even though several RWs fail, the recovery process, on average, does not lead to an undesired increase beyond $\ntarget$. Without forking, the second burst failure would lead to a catastrophic failure. Intuitively, the smaller the graph, measured by $\nnodes$, the faster the reaction time to failure events. This is because $\returncdf$ is confined to a smaller support. 
\cref{fig:rrg_different_eps} depicts the performance of \algname for different values of $\thres$ and $\nnodes=100$ nodes through numerical experiments. Different choices for $\thres$ illustrate the trade-off between reaction time and undesired forks beyond $\ntarget=10$, i.e., objectives~\ref{item:early_detection} and \ref{item:avoid_flood}, respectively. By decreasing the reaction time by choosing larger $\thres$, the number of undesired forks increases.
\begin{figure}[t]
    \centering
    \iffigures
    \resizebox{0.8\linewidth}{!}{\input{plots/different_eps}}
    \fi
    \vspace{-0.1cm}
    \caption{\small \algname on a $8$-degree random regular graph with $\nnodes=100$. Different choices for $\thres$ show the trade-off between reaction time and undesired forks beyond $\ntarget=10$.} %
    \label{fig:rrg_different_eps} \vspace{-.1cm}
\end{figure} 
In \cref{fig:different_graphs}, we show the performance for different types of graphs of the same size, and for the algorithm parameterized by similar values for $\thres\in\{1.9,2,2.1\}$. It can be seen that our strategy performs equally well on those graphs. This is well-justified by the fact that the survival function is estimated at each of the nodes, which respects the actual graph properties at hand, i.e., the actual return time distribution of RWs on such a graph, and does not rely on any assumptions of the distributions.

\begin{figure}[!t]
    \centering
    \iffigures
    \resizebox{.8\linewidth}{!}{\input{plots/different_graphs}}
    \fi
    \caption{\small Stable results for \algname on four different types of graphs with $\nnodes=100$ nodes.}
    \label{fig:different_graphs} \vspace{-.5cm}
\end{figure}

\section{Conclusion}

We introduced two novel decentralized and dynamic algorithms that adapt to failures of RWs on a graph. When detecting failures, \algname allows nodes to fork RWs probabilistically, thus maintaining a desired value of RWs and avoiding catastrophic failures. \algnameplus additionally allows the deliberate termination of RWs. We numerically compare our algorithms to a baseline, termed \missingperson. %
Through simulations and theoretical analysis, we showed a trade-off between the competing objectives of quickly forking RWs and the undesired event of forking RWs in the absence of failures. We show that the number of undesired forks after failures is bounded. Our findings open up many interesting research directions, e.g., analyzing the algorithmic properties in relation to the number of nodes for general graphs.

\appendix

We provide in the following (in this order) the proofs of \cref{lemma:forkedcdf,prop:lower_bound,thm:upper_bound,thm:overshoot}.

\begin{proof}[Proof of \cref{lemma:forkedcdf}]
If at time $\forktime$ a random node $\nodeidx$ forks a random walk, the arrival time $\arrivaltimerv$ of the new random work at a random node $j\in [\nnodes] \setminus \{\nodeidx\}$ is an exponentially distributed random variable with parameter $\arrivalrate$ according to \cref{assumption:distributions}. From the arrival on, the return times can again be assumed to be distributed subject to $\exp(\returnrate)$. 
For a forked but terminated random walk ($\termtime<t$), we have for $\arrivaltime < \termtime$ that $\forkcdfcond[\forkcdfvar] = $ \vspace{-.3cm}
\begin{align*}
    \begin{cases}
        \frac{\forkcdfvar}{e^{-\returnrate (t-\termtime)}}, & 1 - \returncdf[t-\termtime] \geq \forkcdfvar \geq 1 - \returncdf[t-\arrivaltime] \\
        1, & \forkcdfvar > 1 - \returncdf[t-\termtime] \\
        0, & \text{else}.
    \end{cases}
\end{align*}
For $\arrivaltime \geq \termtime$, we have $\forkcdfcond[\forkcdfvar]=1$, i.e., by setting $\termtime = t$. At time $t$, by the total law of probability, we have for $1 - \returncdf[t-\termtime] \geq \forkcdfvar \geq 1 - \returncdf[t-\forktime]$ that
\begin{align*}
    &\forkcdf[\forkcdfvar] = \int_{\forktime}^{\infty} \forkcdfcond[\forkcdfvar] \arrivalpdf[\arrivaltime - \forktime] d\arrivaltime \\
    &\overset{(i)}{=} \!\!\! \int_{\forktime}^{\termtime} \!\!\!\!\! \forkcdfcond[\forkcdfvar] \arrivalpdf[\arrivaltime \! - \! \forktime] d\arrivaltime \!\! + \!\! \int_{t}^{\infty} \!\!\!\!\!\!  \arrivalpdf[\arrivaltime \! - \! \forktime] d\arrivaltime \\
    &= \!\! \int_{\forktime}^{\termtime} \!\!\!\! \mathds{1}\{\forkcdfvar \! > \! 1 \! - \! \returncdf[t-\arrivaltime]\} \frac{\forkcdfvar \arrivalpdf[\arrivaltime - \forktime]}{e^{-\returnrate (t-\termtime)}} d\arrivaltime \! + \! e^{-\arrivalrate (\termtime-\forktime)} \\
    &\overset{(ii)}{=} \int_{\forktime}^{t + \frac{\log(\forkcdfvar)}{\arrivalrate}} \frac{\forkcdfvar \arrivalrate e^{-\arrivalrate (t-\forktime)}}{e^{-\returnrate (t-\termtime)}} d\arrivaltime + e^{-\arrivalrate (\termtime-\forktime)} \\
    &= \frac{\forkcdfvar}{e^{-\returnrate (t-\termtime)}} \left(1 - e^{-\arrivalrate (t + \frac{\log(\forkcdfvar)}{\returnrate} - \forktime)} \right)  + e^{-\arrivalrate (\termtime-\forktime)} \\
    &= \frac{\forkcdfvar}{e^{-\returnrate (t-\termtime)}} \left(1 - e^{-\arrivalrate (t - \forktime)} x^{-\frac{\arrivalrate}{\returnrate}} \right) + e^{-\arrivalrate (\termtime-\forktime)}
\end{align*}
where in $(i)$ we used $\forkcdfcond[\forkcdfvar] = 1$ for $\arrivaltime > t$. 
where $(ii)$ follows from the support of $\forkcdfcond$ as $\forkcdfvar \geq 1-\returncdf[t-\arrivaltime] = e^{-\returnrate (t-\arrivaltime)}$ and thus $\arrivaltime \leq t + \frac{\log(\forkcdfvar)}{\returnrate}$. 
Values of $\forkcdfvar < 1- \returncdf[t-\forktime]$ are impossible to observe, and thus 
we obtain the CDF $\forkcdf[\forkcdfvar]$ given in \cref{lemma:forkedcdf}. We omit the proof of the corollary for brevity. When $\termtime>t$, the result follows by setting $\termtime=t$.
\end{proof}

\begin{proof}[Proof of \cref{prop:lower_bound}]
We derive a worst-case bound on the reaction time to the failure of $\nterm$ random walks at time $\termtime$. The difficulty is that the estimate $\nest$ after the failures of $\nterm$ random walks can no longer be described by the Irwin-Hall distribution since it relies on the fact that all the uniform distributions it sums have the same support. We split the estimation $\nest$ into two parts, where one is from the remaining $\nactive$ random walks, from which $\nactive-1$ have a probabilistic contribution towards $\nest$, and one is for the $\nterm$ terminated random walks. While the former is the sum of $\nactive$ uniformly distributed random variables between $0$ and $1$, the latter is the sum of uniform distributions supported between $0$ and $e^{-\returnrate(t-\termtime)}$. The CDF of this distribution can be described by scaling the CDF in of the Irwin-Hall distribution by the inverse support of the individual uniform distributions, i.e., we can write for the new CDF $F_{\Sigma_{\nterm}}^\prime(x) = F_{\Sigma_{\nterm}}\!\left(\!\frac{x}{\exp(-\returnrate (t-\termtime))}\!\right)$. The probabilistic part of the estimation, i.e., the sum of the estimations for the $\nactive-1$ remaining random walks not visiting node $\nodeidx$ at that time, can be described by $F_{\Sigma_{\nactive-1}}(x)$. Nodes $\nodeidx$ forks with probability $\pforkscale$ when the sum of both partial estimations is below $\thres-\frac{1}{2}$, for which it is sufficient that one of the estimates is smaller than $\epstmp$ and the other is smaller than $\thres-\epstmp-\frac{1}{2}$, for some $\epstmp < \thres-\frac{1}{2}$. Restricting to one particular $\epstmp$ is a worst-case scenario, but bounds the probability of forking from below. Similarly, for ease of exposition, we assume the worst case (for this bound) of only one remaining random walk $\nactive=1$; hence, only one node can make a forking decision at each time $t$. The probability $\deltanofork$ that no fork happens until time $\ttotal$ is then bounded by the product of probabilities that the sum of estimations is above $\thres$, which is bounded from above by $1-\pforkscale \cdot F_{\Sigma_{\nactive-1}}(\epstmp) \cdot F_{\Sigma_{\nterm}}\!\left(\!\frac{\thres-\epstmp-\frac{1}{2}}{\exp(-\returnrate (t-\termtime))}\!\right)\!$. Hence, we have, \vspace{-.3cm}
\begin{align*}
    \deltanofork &\leq \!\! \prod_{t^\prime = \termtime}^\ttotal \! \left[1-\pforkscale \cdot F_{\Sigma_{\nactive-1}}(\epstmp) \cdot F_{\Sigma_{\nterm}}\!\left(\!\frac{\thres-\epstmp-\frac{1}{2}}{\exp(-\returnrate (t-\termtime))}\!\right)\! \right]\!.
\end{align*}
The bound equivalently holds for $\nterm = \bar{\nterm}-\recovered$ when after the failure of $\bar{\nterm}$ RWs already $\recovered < \bar{\nterm}$ were forked/recovered. In this case, the estimates will statistically be strictly less than that if only $\bar{\nterm}-\recovered$ walks had failed from the start, and none of those would be recovered yet, which was assumed for the above statement. 
It remains to choose the smallest $\ttotal$ such that the desired confidence $\deltanoforksol$ is reached. This concludes the proof of the theorem.
\end{proof}

\begin{proof}[Proof of \cref{thm:upper_bound}]
Throughout the proof, we abuse notation for convenience and let $\pforkcurrent[\rwstep] = \pforkcurrent[\rwstep]^+$. 
Assume we have a stable number $\ncurrent$ of random walks in the system that every node saw at least once, the probability of not forking for $\ttotal$ timesteps is lower bounded by $(1-\pforkcurrent)^T$. We bound the probability of $\nwalksrw$ increasing beyond $\nbound$ from above by considering all events that would lead to such an event.
We assume the worst-case scenario that a node forks right at time instance $t=0$, i.e, the number of active random walks increases from $\rwstep = \ntarget$ to $\rwstep + 1$, which minimizes the upper bound. To get a handle on the number of random walks in the system, we ensure that until every node in the graph sees the new random walk with a high probability $1-\deltavisiti$, no other fork occurs with probability $1-\deltadoubleforki$. Obviously, there is a tension between those two probabilities $\deltavisiti$ and $\deltadoubleforki$. The decrease of the one leads to the increase of the other, and vice versa. Hence, there exists an optimal $\deltasumi$ that minimizes their sum. To find the best trade-off, we analyze $\deltavisiti$ and $\deltadoubleforki$.

Assuming that the first passage times for a newly forked random walk are independent across nodes\footnote{According to \cite{tishby2022cover}, the correlation between passage times has only little impact on the cover time.}, we have that with probability at most $\deltavisiti = 1-(1-e^{-\arrivalrate \tallvisiti})^\nnodes$ any node was not visited during time $\tallvisiti$. %
During time $\tallvisiti$, no other fork should occur. Otherwise, the number of walks could potentially increase beyond $\nbound$. The probability of forking during this time period is given as $\deltadoubleforki = 1-(1-\pforkcurrent[\rwstep])^{\tallvisiti}$, which is actually a function of $\deltavisiti$ because of the dependence of $\tallvisiti$ on $\deltavisiti$, i.e., \vspace{-.3cm}
\begin{align*}
    \deltadoubleforki &= 1-(1-\pforkcurrent[\rwstep])^{\frac{-\log(1-\sqrt[\nnodes]{1-\deltavisiti})}{\arrivalrate}} \\
    &= 1-(1-\pforkcurrent[\rwstep])^{\frac{-\log(1-1-e^{-\arrivalrate \tallvisiti})}{\arrivalrate}} = 1-(1-\pforkcurrent[\rwstep])^{\tallvisiti}
\end{align*}
To pick the optimal $\tallvisiti$, we bound the probabilities $\deltavisiti$ and $\deltadoubleforki$ as follows: \vspace{-.2cm}
\begin{align*}
    \deltavisiti &= 1-(1-e^{-\arrivalrate \tallvisiti})^\nnodes \leq \nnodes e^{-\arrivalrate \tallvisiti} \\
    \deltadoubleforki &= 1-(1-\pforkcurrent[\rwstep])^{\tallvisiti} \leq \tallvisiti \pforkcurrent[\rwstep],
\end{align*}
which follows from Bernoulli's inequality $(1+x)^r \geq 1 + rx$, for integer $r\geq1$ and $x\geq-1$. This enables minimizing $\deltavisiti + \deltadoubleforki$ by setting equal the derivatives of the above equations. This results in $\tallvisiti=-\log(\pforkcurrent[\rwstep]/(\arrivalrate \nnodes))/\arrivalrate$. When $\ttotal > \tallvisiti$, we can bound that up until $\tallvisiti$ when only one fork occurs. This can be repeated for consecutive forks until from $\rwstep = \nmax-1$ to $\rwstep+1 = \nmax\leq \nbound$ as long as $\sum_{\rwstep = \ntarget}^{\nmax-1} \tallvisiti < \ttotal$.

When $\ttotal$ is large enough, so that $\nmax=\nbound$, we can not afford another fork, as otherwise the number of walks would increase beyond $\nbound$. Hence, for the remaining time $\tnofork = \ttotal - \sum_{\rwstep = \ntarget}^{\nmax-1} \tallvisiti$ we must guarantee no fork occurs. While doing so, we use that we ensured all nodes have seen all forked random walks. Otherwise, the failure probability is captured in the $\deltavisiti$'s. Hence, the probability of not forking with $\nmax$ active number of nodes conditioned on every node having seen all walks is at most $\deltanoforkstable = 1 - (1-\pforkcurrent[\nbound])^{\tnofork}$. When $\nmax < \nbound$, we assume that we have $\nmax$ active random walks and bound the probability of not forking for the remaining time until $\ttotal$ as $\deltanoforkstable = 1 - (1-\pforkcurrent[\nmax])^{\tnofork}$.

When $\ttotal < \tallvisiti[\ntarget]$, we can bound the failure probability by $\deltanofork \leq 1-(1-\pforkcurrent[\ntarget])^\ttotal$. 
By a union bound over all undesirable events that could ultimately lead to an increase of the number of random walks beyond $\nbound$, which are 1) that not all nodes are aware of a forked random walk after a certain time, 2) another fork happens before all nodes saw the new fork with high probability, and 3) that a critical fork happens when we reached the maximum $\nbound$, i.e., \vspace{-.2cm}
\begin{align*}
        \delta &\leq \! 1 \! - (1 \!- \! \pforkcurrent[\nmax])^{\tnoforki[\nmax]} \!+ \!\!\! \sum_{\forkidxtmp=\ntarget}^{\nmax-1}\!\! 1 \! - \! (1\!-\!e^{-\arrivalrate \tallvisiti})^\nnodes +  1\!-\!(1\!-\!\pforkcurrent[\rwstep])^{\tallvisiti}\!\!. \\[-.3cm]
        &\leq \pforkcurrent[\nmax] \tnoforki[\nmax] + \sum_{\forkidxtmp=\ntarget}^{\nmax-1} \nnodes e^{-\arrivalrate \tallvisiti} + \tallvisiti \pforkcurrent[\rwstep]. \vspace{-.3cm}
    \end{align*}
This concludes the proof of the theorem. The corollary follows by inverting the statement and finding the best $\delta$.
\end{proof}

\begin{proof}[Proof of \cref{thm:overshoot}]
Assume that $\nactive < \ntarget$ random walks are active after one or multiple failure events after time $\termtime$, and assume until time $\firstfork>t$ no fork happened. %
To illustrate the proof of the theorem, we first give an example for $t-\firstfork=3$, where we bound 
$\mathbb{E}[\nwalksrw[\firstfork+3]]$, followed by the proof for any $\mathbb{E}[\nwalksrw[\firstfork+x]]$. For $\mathbb{E}[\nwalksrw[\firstfork+3]]$, we have for any $\nrwtmp_1 \in [\nwalks[\firstfork] , 2\nwalks[\firstfork]]$, $\nrwtmp_{1,0} \in [\nwalks[\firstfork] , 2 \nrwtmp_1]$, $\nrwtmp_{1,1} \in [\nrwtmp_1+1 , 4 \nwalks[\firstfork]]$ that \vspace{-.1cm}
\begin{align*}
    \mathbb{E}[&\nwalksrw[\firstfork+3]] 
    =\!\!\!\!\!\! \begin{aligned}[t] &\sum_{\nwalks[\firstfork+2] = \nrwtmp_1+1}^{4\nwalks[\firstfork]} \mathbb{E}[\nwalksrw[\firstfork+3]\vert \nwalksrw[\firstfork+2] = \nwalks[\firstfork+2], \nwalksrw[\firstfork+1] > \nrwtmp_1] \\[-.1cm]
    &\Pr(\nwalksrw[\firstfork+2] = \nwalks[\firstfork+2] \vert \nwalksrw[\firstfork+1] > \nrwtmp_1) \Pr(\nwalksrw[\firstfork+1] > \nrwtmp_1) + \\
    &+\sum_{\nwalks[\firstfork+2] = \nwalks[\firstfork]}^{\nrwtmp_1} \mathbb{E}[\nwalksrw[\firstfork+3]\vert \nwalksrw[\firstfork+2] = \nwalks[\firstfork+2], \nwalksrw[\firstfork+1] \leq \nrwtmp_1] \\[-.1cm]
    &\Pr(\nwalksrw[\firstfork+2] = \nwalks[\firstfork+2] \vert \nwalksrw[\firstfork+1] \leq \nrwtmp_1) \Pr(\nwalksrw[\firstfork+1] \leq \nrwtmp_1) \end{aligned} \\
    &\overset{(a)}{\leq} \begin{aligned}[t] &\mathbb{E}[\nwalksrw[\firstfork+3]\vert 4\nwalksrw[\firstfork] \geq \nwalksrw[\firstfork+2] > \nrwtmp_{1,1}, \nwalksrw[\firstfork+1] = 2\nwalks[\firstfork]] \\
    &\Pr(4\nwalksrw[\firstfork] \geq \nwalksrw[\firstfork+2]  > \nrwtmp_{1,1} \vert \nwalksrw[\firstfork+1] > \nrwtmp_1) \Pr(\nwalksrw[\firstfork+1] > \nrwtmp_1) \\
    &+ \mathbb{E}[\nwalksrw[\firstfork+3]\vert 2 \nwalks[\firstfork] \leq \nwalksrw[\firstfork+2] \leq \nrwtmp_{1,1}, \nwalksrw[\firstfork+1] = 2\nwalks[\firstfork]] \\
    &\Pr(2 \nwalks[\firstfork] \leq  \nwalksrw[\firstfork+2]  \leq \nrwtmp_{1,1} \vert \nwalksrw[\firstfork+1] > \nrwtmp_1) \Pr(\nwalksrw[\firstfork+1] > \nrwtmp_1) \\
    &+\mathbb{E}[\nwalksrw[\firstfork+3]\vert 2\nrwtmp_1 \geq \nwalksrw[\firstfork+2] > \nrwtmp_{1,0}, \nwalksrw[\firstfork+1] = \nrwtmp_1] \\
    &\Pr(2\nrwtmp_1 \geq \nwalksrw[\firstfork+2] > \nrwtmp_{1,0} \vert \nwalksrw[\firstfork+1] \leq \nrwtmp_1) \Pr(\nwalksrw[\firstfork+1] \leq \nrwtmp_1) \\
    &+ \mathbb{E}[\nwalksrw[\firstfork+3]\vert \nrwtmp_1 \leq \nwalksrw[\firstfork+2] \leq \nrwtmp_{1,0}, \nwalksrw[\firstfork+1] = \nrwtmp_1] \\
    &\Pr(\nrwtmp_1 \leq \nwalksrw[\firstfork+2] \leq \nrwtmp_{1,0} \vert \nwalksrw[\firstfork+1] \leq \nrwtmp_1) \Pr(\nwalksrw[\firstfork+1] \leq \nrwtmp_1) \end{aligned} \\
    &\overset{(b)}{\leq} \begin{aligned}[t] &\mathbb{E}[\nwalksrw[\firstfork+3]\vert \nwalksrw[\firstfork+2] = 4\nwalks[\firstfork], \nwalksrw[\firstfork+1] = 2\nwalks[\firstfork]] \\
    &\Pr(\nwalksrw[\firstfork+2]  > \nrwtmp_{1,1} \vert \nwalksrw[\firstfork+1] = 2\nwalksrw[\firstfork]) \Pr(\nwalksrw[\firstfork+1] > \nrwtmp_1) + \\
    &+ \mathbb{E}[\nwalksrw[\firstfork+3]\vert \nwalksrw[\firstfork+2] = \nrwtmp_{1,1}, \nwalksrw[\firstfork+1] \! = \! 2\nwalks[\firstfork]] \Pr(\nwalksrw[\firstfork+1] \!\! > \!\! \nrwtmp_1) \\
    &+\mathbb{E}[\nwalksrw[\firstfork+3]\vert \nwalksrw[\firstfork+2] = 2\nrwtmp_1, \nwalksrw[\firstfork+1] = \nrwtmp_1] \\
    &\Pr(\nwalksrw[\firstfork+2] > \nrwtmp_{1,0} \vert \nwalksrw[\firstfork+1] = \nrwtmp_1) \\
    &+ \mathbb{E}[\nwalksrw[\firstfork+3]\vert \nwalksrw[\firstfork+2] = \nrwtmp_{1,0}, \nwalksrw[\firstfork+1] = \nrwtmp_1] \end{aligned}
\end{align*}
where $(a)$ is because the conditional expectation of the number of random walks can only increase by increasing one of the previous realizations, and by splitting the summation over realizations of $\nwalks[\firstfork+2]$ into two disjoint parts. $(b)$ holds since $\Pr(\nwalksrw[\firstfork+x] > \nrwtmp_x \vert \nwalksrw[\firstfork+x-1] \leq \nrwtmp_{x-1}, \cdots, \nwalksrw[\firstfork+1] \leq \nrwtmp_1) \leq \Pr(\nwalksrw[\firstfork+x] \vert \nwalksrw[\firstfork+x-1] = \nrwtmp_{x-1}, \cdots, \nwalksrw[\firstfork+1] = \nrwtmp_1)$ for any $\nrwtmp_{x} \geq \nrwtmp_{x-1} \cdots \geq \nrwtmp_{1}$.

To bound the expected number of random walks at any point $\firstfork+x$ in time, we observe that the above proof strategy exhibits the structure of a binary tree, where the root is given by the threshold $\nrwtmp_1$ for time step $\firstfork+1$. The following level of hierarchy is given by the two thresholds $\nrwtmp_{1,0}$ and $\nrwtmp_{1,1}$ for time step $\firstfork+2$. The last level contains the leaf nodes. In general, the tree has a depth of $x-1$, and each of the $2^{x-1}$ leaves determines a conditional expectation associated with a probability. We use a binary labeling of the tree, and hence the thresholds. Each level $x^\prime < x$ consists of $2^{x\prime-1}$ thresholds $\nrwtmp_{1, \mathbf{a}}$, one for every binary vector $\mathbf{a} \in \{0,1\}^{x-2}$ of length $x-2$. It is crucial that $\nrwtmp_{1, \mathbf{a}[:x^\prime], 1} > \nrwtmp_{1, \mathbf{a}[:x^\prime]}$ and $\nrwtmp_{1, \mathbf{a}[:x^\prime], 0} \leq 2\nrwtmp_{1, \mathbf{a}[:x^\prime]}$ for every $\mathbf{a}$ and $x^\prime < x-2$. Each binary vector $\mathbf{a} \in \{0,1\}^{x-1}$ describes a path through the tree. Traversing the tree from top to bottom, going left in level $x^\prime$ at branch $1, \mathbf{a}[:x^\prime-1]$, i.e., $\mathbf{a}[x^\prime]=0$ means $\nwalks[\firstfork+x^\prime] \leq \nrwtmp_{1, \mathbf{a}[:x^\prime-1]}$. The threshold $\nrwtmp_{1, \mathbf{a}[:x^\prime-1]}$ should be designed so this happens with probability close to $1$. Going right means that $\nwalks[\firstfork+x^\prime] > \nrwtmp_{1, \mathbf{a}[:x^\prime-1]}$, and hence we have to assume the worst case that $\nwalks[\firstfork+x^\prime] = 2\nwalks[\firstfork+x^\prime-1]$. This should only happen with sufficiently small probability, i.e., $\Pr(\nwalks[\firstfork+x^\prime] > \nrwtmp_{1, \mathbf{a}[:x^\prime-1]} \vert \history[\firstfork+x^\prime](\mathbf{a}[:x^\prime-1])) \approx 0$, where $\history[\firstfork+x^\prime](\mathbf{a}[:x^\prime-1])$ is the history conditioned on path $\mathbf{a}$. We drop the dependency on $\mathbf{a}$ for clarity of presentation. Starting from $\nwalks[\firstfork]$, the number of random walks in the history for $\nwalks[\firstfork+x^\prime]$ is  $2\nwalks[\firstfork+x^\prime-1]$ if $\mathbf{a}[x^\prime]=1$ and $\nrwtmp_{1,\mathbf{a}[:x^\prime-1]}$ if $\mathbf{a}[x^\prime]=0$. The probability of $\Pr(\nwalks[\firstfork+x^\prime] > \nrwtmp_{1, \mathbf{a}[:x^\prime-1]} \vert \history[\firstfork+x^\prime])$ can be bounded from above since we have from \cref{thm:upper_bound} an upper bound on the forking probability at each time step conditioned on the history. We can write \vspace{-.2cm} %
\begin{align*}
    &\Pr(\nwalksrw[\firstfork + x^\prime] > \nwalks[\firstfork + x^\prime] | \nwalksrw[\firstfork+x^\prime-1] = \nwalks[\firstfork+x^\prime-1], \history[\firstfork+x^\prime-1]) \\
    & \leq \!\!\!\!\!\! \sum_{\nrwtmp = \nwalks[\firstfork+x^\prime-1]}^{2 \nwalks[\firstfork+x^\prime-1]} \!\!\!\! \binom{\nwalks[\firstfork+x^\prime-1]}{\nrwtmp -\nwalks[\firstfork+x^\prime-1]} \forkboundt[\firstfork+x^\prime]^{\nrwtmp-\nwalks[\firstfork+x^\prime-1]} \\[-.4cm]
    &\hspace{3.5cm} \left(1-\forkboundt[\firstfork+x^\prime]\right)^{2\nwalks[\firstfork+x^\prime-1] - \nrwtmp}.
\end{align*}
Hence, we obtain for each path $\mathbf{a} \in \{0,1\}^{x-1}$ a history $\history[\firstfork+x]$, which determines by the chain rule the probability of taking that path as the product of conditional probabilities. Since we can only upper bound $\Pr(\nwalks[\firstfork+x^\prime] > \nrwtmp_{1, \mathbf{a}[x^\prime-1]} \vert \history[\firstfork+x^\prime])$, we bound $\Pr(\nwalks[\firstfork+x^\prime] \leq \nrwtmp_{1, \mathbf{a}[x^\prime-1]} \vert \history[\firstfork+x^\prime]) \leq 1$, which is reasonable given an appropriate choice of $\nrwtmp_{1, \mathbf{a}}$. Having established the history for each path, we can bound the corresponding probability and the expected value $\mathbb{E}[\nwalks[\firstfork+x] \vert \history[\firstfork+x]]$ as \vspace{-.2cm}
\begin{align}
    \mathbb{E}&[\nwalksrw[\firstfork+x] \vert \nwalksrw[\firstfork+x-1]=\nwalks[\firstfork+x-1], \history[\firstfork+x-1]] \nonumber \\
    &\leq \nwalks[\firstfork+x-1] + \nwalks[\firstfork+x-1] \forkboundt[\firstfork+x-1], \label{eq:upper_bound_expectation}\vspace{-.5cm} 
\end{align}
where we separate $\nwalksrw[\firstfork+x-1]=\nwalks[\firstfork+x-1]$ from the remaining history for clarity of presentation in bounding the conditional expectation. 
In total, we obtain $2^{x-1}$ of such conditional expectations and associated probabilistic upper bounds, which allows to compute an upper bound on $\mathbb{E}[\nwalksrw[\firstfork+x]]$ by summing over all $\mathbf{a} \in \{0,1\}^{x-1}$ the upper bound on the expectation
$\mathbb{E}[\nwalksrw[\firstfork+x] \vert \nwalksrw[\firstfork+x-1]=\nwalks[\firstfork+x-1], \cdots, \nwalksrw[\firstfork+1]=\nwalks[\firstfork+1], \nwalksrw[\firstfork]=\nwalks[\firstfork]]$ according to \eqref{eq:upper_bound_expectation}, with $\nwalks[\firstfork+x^\prime-1]=2\nwalks[\firstfork+x^\prime-2]$ if $\mathbf{a}[x^\prime-1]=1$ and $\nwalks[\firstfork+x^\prime-1]=\nrwtmp_{1,\mathbf{a}[:x^\prime-2]}$ if $\mathbf{a}[x^\prime-1]=0$, times its probabilistic upper bound $\prod_{x^\prime=1}^x \mathds{1}\{\mathbf{a}[x^\prime]=1\} \Pr(\nwalksrw[\firstfork+x^\prime] > \nrwtmp_{1,\mathbf{a}[:x^\prime-1]} \vert \nwalksrw[\firstfork+x^\prime-1]=\nwalks[\firstfork+x^\prime-1], \history[\firstfork+x^\prime-1])$, where $\forall x^\prime \in [2, x]: \nwalks[\firstfork+x^\prime-1] = \left(\mathds{1}\{\mathbf{a}[x^\prime-1]=0\} \nrwtmp_{\mathbf{a}[:x^\prime-2]}+\mathds{1}\{\mathbf{a}[x^\prime-1]=1\} 2 \nwalks[\firstfork+x^\prime-2] \right)$. This concludes the proof of \cref{thm:overshoot}.
\end{proof}

\bibliographystyle{IEEEtran}
\bibliography{Refs}

\IEEEtriggeratref{4}

\end{document}